%% file: main.tex
\documentclass{article}

\PassOptionsToPackage{numbers, compress}{natbib}
\usepackage[preprint]{neurips_2025}

\usepackage[utf8]{inputenc} % allow utf-8 input
\usepackage[T1]{fontenc}    % use 8-bit T1 fonts
\usepackage{hyperref}       % hyperlinks
\usepackage{url}            % simple URL typesetting      % professional-quality tables
\usepackage{amsfonts}       % blackboard math symbols
\usepackage{nicefrac}       % compact symbols for 1/2, etc.
\usepackage{enumitem}
\usepackage{microtype}      % microtypography
\usepackage{xcolor}         % colors

\usepackage{booktabs} % For formal tables
\usepackage{makecell}
\usepackage{textcomp}
\usepackage{amsmath, amsthm}

\newtheorem{theorem}{Theorem}
\usepackage[ruled, vlined, linesnumbered]{algorithm2e}
\usepackage[noend]{algpseudocode}
\usepackage{pdfpages}
\usepackage{graphicx}
\usepackage{amssymb}
\usepackage{mathrsfs}
\usepackage{epsfig}
\usepackage{graphicx}
\usepackage{multirow}
\DeclareMathOperator*{\argmax}{arg\,max}

\usepackage{balance}
\usepackage{wrapfig}
\usepackage[normalem]{ulem}
\usepackage{threeparttable}
\newcolumntype{C}[1]{>{\centering\arraybackslash}p{#1}}
\usepackage{lineno}

\usepackage{bbm}

\usepackage{graphicx}
\usepackage{float}
\usepackage{subfig}

\def \bx {\mathbf{x}}

\def \bbE {\mathbb{E}}

\def \br {\boldsymbol{r}}
\def \bg {\boldsymbol{g}}

\def \bs {\boldsymbol{s}}
\def \be {\boldsymbol{e}}
\def \bx {\boldsymbol{x}}
\def \bc {\boldsymbol{c}}
\def \bG {\boldsymbol{G}}
\def \cD {\mathcal{D}}

\def \cN {\mathcal{N}}

\def \cL {\mathcal{L}}
\def \cA {\mathcal{A}}
\def \cV {\mathcal{V}}
\def \bM {\boldsymbol{M}}

\def \Indi {\mathbbm{1}}
\def \Pr {\mathrm{Pr}}

\title{Goal-Conditioned Supervised Learning for Multi-Objective Recommendation}

\author{
   Shijun Li$^{1}$,
   Hilaf Hasson$^{2}$,
   Jing Hu$^{3}$, 
   Joydeep Ghosh$^{4}$ 
   \\
  $^{1,4}$The University of Texas at Austin, $^{2}$Intuit AI Research, $^{3}$Intuit Inc.\\
  \texttt{\{shijunli, jghosh\}@utexas.edu, \{hilafhasson, hu.jing000\}@gmail.com}\\
  \ \ \
  }

\begin{document}
% \begin{sloppypar}
\maketitle

\begin{abstract}

Multi-objective learning endeavors to concurrently optimize multiple objectives using a single model, aiming to achieve high and balanced performance across diverse objectives. However, this often entails a complex optimization problem to balance the learning of potentially conflicting objectives, leading to solutions with higher memory requirements and computational complexity. This paper introduces a Multi-Objective Goal-Conditioned Supervised Learning (MOGCSL) framework for automatically learning to achieve multiple objectives from offline sequential data. MOGCSL extends the conventional GCSL method to multi-objective scenarios by redefining goals from one-dimensional scalars to multi-dimensional vectors. It benefits from naturally eliminating the need for complex architectures and optimization constraints. Moreover, MOGCSL inherently disentangles uninformative or noisy training instances that fail to achieve desirable long-term rewards across multiple objectives. We also introduce a novel goal-selection algorithm for MOGCSL to model and identify desired and achievable goals for inference.

In this paper, we focus on its application to the next action prediction problem in commercial-grade recommender systems. In this context, any viable solution needs to be reasonably scalable and also be robust to large amounts of noisy data that is characteristic of this application space. We show that MOGCSL performs admirably on both counts by extensive experiments. Also, analysis and experiments are included to explain its strength in discounting the noisier portions of training data in recommender systems with multiple objectives. Code is available at: \url{https://github.com/xiwenchao/MOGCSL}.

\end{abstract}

\input{sections/1.introduction}

\input{sections/2.related}

% \input{sections/4.framework}

\input{sections/5.method}

\input{sections/6.experiment}

\section{Conclusion}

In this work, we propose a novel framework named MOGCSL for multi-objective recommendation. MOGCSL utilizes a vectorized goal to disentangle the representation of different objectives. It can be directly combined with conventional sequential models and optimized through supervised learning, without requiring handcrafted model architecture or optimization constraints. Beyond training process, we theoretically analyze the properties of inference goals and propose a novel goal-generation algorithm accordingly. Extensive experiments demonstrate the superiority of MOGCSL in both effectiveness and efficiency. Future work can extend MOGCSL to other application scenarios.

Despite its effectiveness and efficiency, MOGCSL has several limitations that motivate future work. These include improving CVAE-based goal selection when high-goal trajectories are sparse through data augmentation or iterative online training, explicitly incorporating negative feedback into goal modeling, and validating practical benefits through live-user A/B testing. Extending the framework beyond recommendation to other multi-objective learning tasks is another promising direction. We've include more detailed discussions of limitation and future works in Appendix \ref{sec:limit}.

% For future work, it would be interesting to test our approach in other application scenarios beyond recommender systems.

\section{Acknowledgment}
This work was supported by the Intuit University Collaboration Program.

%\section{Supplementary Material}

%Authors may wish to optionally include extra information (complete proofs, additional experiments and plots) in the appendix. All such materials should be part of the supplemental material (submitted separately) and should NOT be included in the main submission.

\bibliographystyle{unsrtnat}
\bibliography{reference}

\newpage
\appendix
\input{sections/appen_proof}

\input{sections/appen_expdetail}

\input{sections/appen_discussion}

\input{sections/appen_limit}

\newpage
\input{sections/checklist}

\end{document}

%% file: sections/1.introduction.tex
\section{Introduction}\label{sec:intro}

Multi-objective learning techniques
% has been recognized and demonstrated as an effective strategy to address the prevalent challenge of learning for multiple objectives. It 
typically aim to train a single model to determine a policy for multiple objectives, which are often defined on different types of tasks \cite{ruder2017overview, sener2018multi, zhang2021survey}.  % ruder2017overview, sener2018multi,
For instance, two common tasks in recommender systems are to recommend items that users may click or purchase, and hence the corresponding two objectives are defined as pursuing higher click rate and higher purchase rate. %Multi-objective learning is therefore occasionally conflated with “multi-task learning” in the literature. 
However, learning for multiple objectives simultaneously is often nontrivial, particularly when there are potential inherent conflicts among these objectives \cite{sener2018multi}.

Existing approaches to multi-objective learning broadly address this optimization issue by formulating and optimizing a loss function that takes into account multiple objectives in a supervised learning paradigm. % There have also been attempts to apply longer-term reinforcement learning in the multi-objective space, but have challenges that are difficult to surmount for application in real-world recommender systems (see Section \ref{sec:related} for details).
% as a supervised classification task, that learns the next action, while optimizing a loss function that takes into account the multiple objectives. 
%Models whose loss functions are a combination of loss functions for each of the tasks 
Some research focused on designing model architectures specifically for multi-objective scenarios, such as the shared-bottom structure \cite{ma2018modeling}, the experts and task gating networks \cite{ma2018modeling}, and so on. Another line of research studies how to constrain the optimization process based on various assumptions regarding how to assign reasonable loss weights \cite{liu2019end} or dynamically adjust gradients \cite{yu2020gradient}. 
% \citet{lin2019pareto} proposes an optimization algorithm within the context of Pareto efficiency. 
We remark that these approaches often introduce substantial space and computational complexity \cite{ma2018modeling, zhang2021survey}. Also, they treat all the data samples uniformly. %  (m-estimation).
In recommender systems, sometimes none of the items shown to a user is of interest, rendering their choice uninformative. Also users are often distracted or their interests temporarily change. These are some of the reasons that the observed interaction data in real settings suffers from substantial uninformative or noisy components that are better left discounted. Existing approaches do not cater well to this aspect of our focus application.

To address these issues, we propose a novel method called \textbf{Multi-Objective Goal-Conditioned Supervised Learning (MOGCSL)}. In this framework, we first apply goal-conditioned supervised learning (GCSL) \cite{liu2022goal} to the multi-objective recommendation problem by introducing a multi-dimensional goal. At its core, MOGCSL aims to learn primarily from the behaviors of those sessions where the long-term reward ends up being high, thereby discounting noisy user choices coming with low long-term rewards. 
 % patterns with desirable high goals, thereby discarding noisy user choices coming with low goals. 
 Unlike conventional GCSL, however, we represent the reward gained from the environment with a vector instead of a scalar. Each dimension of this vector indicates the reward for a certain objective. The ``goal'' in MOGCSL can then be defined as a vector of desirable cumulative rewards on all of the objectives given the current state. %Following the offline training paradigm of GCSL, multiple objectives can be simultaneously optimized in standard supervised learning manner, aiming to reach a specified goal by generating a sequence of actions. 
%In this way, the aforementioned optimization quandary can be transformed into a sequence generation problem to achieve desirable goals. 
%That could be achieved by directly maximizing the likelihood of executing the ground-truth actions demonstrated in the offline data. 
%The complex design of model architecture or optimization constraints for multi-objective learning is eliminated. More importantly, 
By incorporating these goals as input, MOGCSL learns to rely primarily on high-fidelity portions of the data with less noise corresponding to multiple objectives, which helps to better predict users' real preference.
% (see Section \ref{sec:theor} for details). 
%We conduct 
Extensive experiments indicate that  MOGCSL significantly outperforms other baselines and benefits from lower complexity.

For inference of GCSL, most previous works employ a simple goal-choosing strategy for scalar goals \cite{DecisionTransformer, xin2022rethinking} (e.g., some multiple of the average goal in training).
To study how to choose reasonable and desirable goals within multi-objective scenarios, 
% Although we observe that simple choices for MOGCSL do reasonably well in practice, we also 
we introduce a goal-choosing algorithm
that estimates the distribution of achievable goals using variational auto-encoders, and automates the selection of desirable high goals on multiple objectives as input for inference. 
This algorithm addresses the key challenge of selecting achievable ``high'' multi-dimensional goals, a problem that has remained unexplored yet. 
By comparing the proposed method with simpler statistical strategies, we gain valuable insights into the goal-choosing process and trade-offs for practical implementation.

Our key contributions can be summarized as follows:

\begin{itemize}[leftmargin=*]
    \item We introduce a supervised framework called MOGCSL for multi-objective recommendations that, by design, 
    % excludes data from users who did not achieve desirable long-term rewards. 
    selectively leverages training data with desirable long-term rewards on multiple objectives for inference.
     We implement this approach using a transformer encoder optimized  with a standard cross-entropy loss,
    avoiding more complex architectures or optimization constraints that are typical for multi-objective learning. Empirical experiments on real-world e-commerce datasets demonstrate the superior efficacy and efficiency of MOGCSL.
    % We comprehensively analyze the working mechanism of our approach, revealing its capability to remove harmful effects of potentially highly noisy samples in training data, see Section \ref{sec:theor}.%By representing goals as vectors, where each dimension corresponds to cumulative rewards for a specific objective, our model can be trained using standard supervised learning techniques and conventional architectures (e.g., transformers) without needing any special design of model structure or learning constraints. 

    \item As a part of MOGCSL, we conduct a formal analysis of goal properties for inference with a theorem. Then we introduce a goal-choosing algorithm that can model the distribution of achievable goals over interaction sequences and choose  desirable ``high''  goals across multiple objectives. % This is novel even for the single-objective case and addresses a fundamental challenge in GCSL.
    % for inference.
    % , leading to superior performance.
    % See Section PART OF EXPERIMENT SECTION THAT DEALS WITH THIS.
    % We empirically evaluate both simple statistical strategies and our goal-generation approach to decide inference goals, yielding valuable insights.
    % \item We implement MOGCSL with an attention-based sequential recommendation model by simply relabeling the offline data with vectorized goals. Empirical experiments on real-world e-commerce datasets demonstrate the superior efficacy and efficiency of MOGCSL. All codes and data will be released if being accepted to facilitate further research.

    \item % We comprehensively analyze the working mechanism of MOGCSL, innovatively revealing its capability to remove harmful effects of potentially  noisy instances in the training data.
    We conduct a comprehensive analysis of MOGCSL's working mechanism and are the first to reveal its ability to effectively mitigate harmful effects of noisy instances in the training data, which is a crucial challenge in practical applications like recommender systems.

\end{itemize}

%% file: sections/2.related.tex
\section{Related Work}\label{sec:related}

\noindent \textbf{Multi-Objective Learning.}
Multi-objective learning typically investigates the construction and optimization of models that can simultaneously achieve multiple objectives. Existing research focuses mainly on resolving the problem by model architecture designs \cite{ma2018modeling, misra2016cross} and optimization constraints \cite{liu2019end, lin2019pareto}. 
All of these works give equal weight to all instances in the training data that they learn from, instead of forcefully distinguishing noisy data from non-noisy data by considering their different effects on multiple objectives. 
%In Section \ref{sec:theor} we will see that a key part of what makes our solution work well is that it does precisely that. 
This is fine in many applications but problematic in commercial recommendation systems. 
Moreover, these approaches often introduce substantial space and computational complexity \cite{zhang2021survey}, making them more challenging for large-scale applications in the real world.
For example, MMOE \cite{ma2018modeling} requires constructing separate towers for each objective. % significantly increasing the model size. 
DWA \cite{liu2019end} necessitates recording and calculating loss change dynamics for each training epoch, while PE \cite{lin2019pareto} demands substantial computational resources to solve an optimization problem for Pareto efficiency. A recent work FAMO \cite{liu2024famo} aims to improve the efficiency of multi-objective learning, but still suffers from the complexity of calculating gradient similarities to determine task weights.

\noindent \textbf{Goal-Conditioned Supervised Learning.}
In contrast, we propose resolving the multi-objective optimization dilemma within the framework of GCSL \cite{liu2022goal, DecisionTransformer, TrajectoryTransformer}. 
% The definition of goals in GCRL can be quite general \cite{liu2022goal}, such as specifying the desired state \cite{chane2021goal, pitis2020maximum, plappert2018multi} or instruction sentences by natural language \cite{chan2019actrce, akakzia2020grounding, colas2020language}. We focus on a specific approach where the goal is defined as the cumulative reward gained along the trajectory \cite{liu2022goal, DecisionTransformer, TrajectoryTransformer, zheng2022online}. 
Typically, GCSL can be directly combined with various sequential models with minor adaptations and trained entirely on offline data. This paradigm effectively transforms offline reinforcement learning into a supervised learning problem. However, as far as we know, most existing works focus on optimizing a single objective. Our work extends GCSL to the multi-objective setting, eliminating the need for scalarization functions or other constraints during training. 
Additionally, although some works have explored how to assign more valuable goals to enhance GCSL training \cite{ajay2020opal, zhuang2024reinformer}, the properties and determination inference goals, especially multi-dimensional ones, remain less explored. We propose a novel algorithm to model the achievable goals and automatically choose desirable goals on multiple objectives as input during inference stage. 
Note that we are solving for next action prediction problem, and use long-term rewards only as extra information, unlike multi-objective reinforcement learning approaches that aim to maximize cumulative returns across multiple objectives \cite{cai2022constrained, stamenkovic2022choosing}. Also, their evaluation principles are different, typically relying on long-term metrics and synthetic environments. % Hence we don't compare with such approaches in this paper. 

% It's worth noting that a recent study \cite{liu2023multi} shares a similar objective with our work, focusing on modeling long-term effects to enhance estimation of immediate feedback. However, their approach utilizes an online RL method that requires iterative interactions with an online environment (real or simulated), whereas MOGCSL is trained entirely on offline data.

%% file: sections/5.method.tex
\section{Methodology}

% In this section, we'll first illustrate the general optimization objective of MOGCSL, making a comparison and discussion with the optimization methods for multi-objective learning. Then, we introduce the training process of MOGCSL on conventional offline sequential data with supervised learning. For inference stage, we first give a formal analysis of the property of achievable goals, which are demonstrated to following a distribution jointly decided by the initial state, the given goal and the action policy. Based on that, we propose a goal-generation framework and empirically explore the effectiveness of simple statistical approaches and advanced generative models to generate the inference goal.  

In this section, we first illustrate the general optimization paradigm of MOGCSL. Then we expound on the training process of MOGCSL and the proposed goal-choosing algorithm for inference. Furthermore, we give a detailed analysis of the denoising capability of MOGCSL.

\subsection{A New View for Multi-Objective Learning}\label{sec:new_view}

Multi-objective learning is typically formulated as an optimization problem over multiple losses \cite{ma2018modeling, misra2016cross, yu2020gradient, lin2019pareto}, each defined on a distinct objective. Consider a dataset $\mathcal{D}=\{(\mathbf{x}_i, y^1_i, y^2_i, ..., y^n_i)\}$, where $\mathbf{x}_i$ represents the feature and $y^j_i$ is the ground-truth score on the $j$-th objective. For a given model $f(\mathbf{x}; \boldsymbol{\theta})$, multiple empirical losses can be computed, one per objective as  $\cL_j(\boldsymbol{\theta})=\bbE_{(\bx, y^j) \in \cD}[\cL(f(\bx; \boldsymbol{\theta}), y^j)]$. 
The model can then be optimized by minimizing a single loss, which is obtained by combining all the losses through a weighted summation as: \: $\min_{\boldsymbol{\theta}} \sum_{j=1}^n w^j \cL_j(\boldsymbol{\theta})$.
% \begin{equation}\small
% \min_{\boldsymbol{\theta}} \quad \sum_{j=1}^n w^j \cL_j(\boldsymbol{\theta})
% \end{equation}

A fundamental question is how to assign these weights and how to regulate the learning process to do well on all the objectives concurrently. Earlier research sought to address this issue based on assumptions regarding the efficacy of certain model architectures or optimization constraints, which may not be generally valid and can significantly increase complexity  \cite{zhang2021survey}.

In contrast, we propose to approach the learning and optimization for multi-objective learning from a different perspective. Specifically, we posit that the interaction process between the agent and the environment can be formalized as a Multi-Objective Markov Decision Process (MOMDP) \cite{roijers2013survey}. 
Denote the interaction trajectories collected by an existing agent as $\mathcal{D}=\{\tau_i\}_{i \in [1,M]}$. In the context of recommender systems, each trajectory $\tau$ records a complete recommendation session between a user entering and exiting the recommender system, such that $\tau=\{(\bs_t, a_t, \br_t)\}_{t \in [1,|\tau|]}$.  A state $\bs_t \in \mathcal{S}$ is taken as the representation of user's preferences at a given timestep $t$. An action $a_t$ is recommended from the action space $\mathcal{A}$ which includes all candidate items, such that $|\cA|=|\cV|=N$ where $\cV$ denotes the set of all items. $R(s_t, a_t)$ is the reward function, where $\br_t=R(s_t, a_t)$ means the agent receives a reward
$\br_t$ after taking an action $a_t$ under state $\bs_t$. Note that the reward function $R(\bs_t, a_t)$ in MOMDP is represented by a multi-dimensional vector instead of a scalar.
% $|\tau|$ is the length of this trajectory, indicating that the user continued interacting with the recommender system from $t=1$ until they terminated this session at time $t=|\tau|$. 

% In this context, each trajectory within the offline sequential data serves as a demonstration of achieving specific performance across multiple objectives. 
In this context, all the objectives can be quantified using reward $\br_t$ at each time step. Specifically, since $\br_t$ is determined by user's behavior in response to recommended items, it naturally reflects the recommender's performance on these objectives. For example, if the user clicks the recommended item $a_t$, the value on the corresponding dimension of $\br_t$ can be set to 1; otherwise, it remains 0.
% Let’s consider two common tasks in recommender systems: click and purchase. The reward  $\br_t$ can be represented as a two-dimensional vector $\br_t=[r^c_t,r^p_t]$, where $r^c_t=1$ if the user clicks the recommended item $a_t$, otherwise $r^c_t=0$. Similarly, $r^p_t$ follows the same definition for purchase events.
In sequential recommendation scenarios, the target of the agent is to pursue better performance at the session level. Session-level performance can be evaluated by the cumulative reward as:
% from the current timestep to the end of the trajectory:
\begin{equation}
\label{eq:goal}\small
\bg_t=\sum_{t'=t}^{|\tau|}  \br_{t'} ,
% \lambda^{t'-t} *
\end{equation}
where $\bg_t$ can be called as a ``goal'' in the literature of goal-conditioned supervised learning \cite{yang2022rethinking}.

Then, the target of multi-objective learning for recommender systems can be formulated as \emph{developing a policy that achieves satisfactory performance 
(i.e., the goals) across multiple objectives in recommendation sessions}. In this research, we address this problem within the framework of GCSL. During the training stage, the aim is to determine the optimal action to take from a given current state in order to achieve the specified goal. The agent, denoted as $\pi_{\theta}$, is trained by maximizing the likelihood of trajectories in the offline training dataset $\cD_{tr}$ through an autoregressive approach, expressed as $\argmax_{\theta} \mathbb{E}_{\cD_{tr}} [log\pi_{\theta}(a|\bs,\bg)] $. Notably, there are no predefined constraints or assumptions governing the learning process. During the inference stage, when an achievable and desirable goal is specified, the model is expected to select an action based on the goal and the current state, with the aim of inducing behaviours to achieve that goal.

\subsection{MOGCSL Training}\label{sec:training}

The initial step of MOGCSL training is relabeling the training data by substituting the rewards with goals. Specifically, for each trajectory $\tau \in \cD_{tr}$, we replace every tuple $(\bs_t, a_t, \br_t)$ with $(\bs_t, a_t, \bg_t)$, where $\bg_t$ is  defined according to Eq. (\ref{eq:goal}). Subsequently, we employ a sequential model  \cite{kang2018self} based on Transformer-encoder (denoted as \emph{T-enc}) to encode the users' sequential behaviors and obtain state representations. We chose a transformer-based encoder due to its widely demonstrated capability in sequential recommendation scenarios \cite{kang2018self, li2023strec}. 
%However, other encoders, such as GRU or CNN, can also be used. 
Specifically, let the interaction history of a user up to time $t$ be denoted as $v_{1:t-1} = \{v_1, ..., v_{t-1}\}$. We first map each item $v \in \cV$ into the embedding space, resulting in the embedding representation of the history: $\be_{1:t-1} = [\be_1, ..., \be_{t-1}]$. Then we encode $\be_{1:t-1}$ by \emph{T-enc}. Since the current timestep $t$ is also valuable for estimating user's sequential behavior, we incorporate it via a timestep embedding denoted as $emb_t$ through a straightforward embedding table lookup operation.
Similarly, we derive the embedding of the goal $emb_{\bg_t}$ through a simple fully connected (FC) layer.
The final representation of state $\bs_t$ is derived by concatenating the sequential encoding, timestep embedding and goal embedding together:
\begin{equation}\small
\label{eq:transf}
emb_{\bs_t}=Concat(\emph{T-enc}(\be_{1:t-1}), emb_t, emb_{\bg_t}).
\end{equation}

\begin{algorithm}[t]\small
 \DontPrintSemicolon
  \textbf{Input:} training data $\cD_{tr}$, batch size $B$, model parameters $\theta$ \; 
  \textbf{Intialization:} initialize parameters $\theta$  \; 
  Relabel all the rewards with goals according to Eq. (\ref{eq:goal}) \;
  \Repeat{convergence}{
     Randomly sample a batch of $(\bs_t, a_t, \bg_t)$ from $\cD_{tr}$ \;
     Compute the representation $\bM_{\bs_t, \bg_t}$ via Eq. (\ref{eq:transf})-Eq.(\ref{eq:atten}) \;
     Derive the prediction logits via Eq. (\ref{eq:ouput}) \;
     Calculate the loss function $ \cL(\theta)$ via Eq. (\ref{eq:loss}) \;
     Update $\theta$ by minimizing $ \cL(\theta)$ with stochastic gradient descent: $\theta \leftarrow \theta- \eta \frac{\partial  \cL(\theta)}{\partial \theta}$
    }
  \caption{Training of MOGCSL}
\label{alg:training}
\end{algorithm}
\vspace{-1mm}

To better capture the mutual information, we feed the state embedding into a self-attention block:
% to get the final representation of each $(\bs_t, \bg_t)$ pair:
\begin{equation}\small
\label{eq:atten}
\bM_{\bs_t, \bg_t}=Atten(emb_{\bs_t}).
\end{equation}

% Since the functionality of the policy network is to decide which action to take (i.e., which item to recommend) given the current state and goal, 
Then we use an MLP to map the derived embedding into the action space, where each logit represents the preference of taking a specific action (i.e., recommending an item):
\begin{equation}\small
\label{eq:ouput}
 [\pi_\theta(v^1|\bs_t, \bg_t),...,\pi_\theta(v^{N} | \bs_t, \bg_t)]   =\mathbf{\delta}(MLP( \bM_{\bs_t, \bg_t})),
\end{equation}

where $v^i$ denotes the $i$-th item in the candidate pool, $\mathbf{\delta}$ is the soft-max function, and $\theta$ denotes all parameters of this model. The model structure is shown in Appendix \ref{fig:structure}.

The training objective is to correctly predict the subsequent action that is mostly likely lead to a specific goal given the current state. As discussed in Section \ref{sec:new_view}, each trajectory of user's interaction history represents a successful demonstration of reaching the goal that it actually achieved. As a result, the model can be naturally optimized by minimizing the expected cross-entropy as:
\begin{equation}\small
\label{eq:loss}
% \cL(\theta) = \bbE_{(\bs_t, a_t, \bg_t) \in \cD}[ - \sum_{i=1}^N  y_{v_i} log(\pi(v_i|\bs_t, \bg_t))],
\cL(\theta) = \bbE_{(\bs_t, a_t, \bg_t) \in \cD_{tr}}[ - log(\pi_\theta(a_t|\bs_t, \bg_t))].
\end{equation}

% where $y_{v_i}$ is a binary label such that $y_{a_t}=1$ and $y_{v_j}=0$ for any $v_j \neq a_t$. 
% The training process is illustrated in Algorithm \ref{alg:training}.

\subsection{MOGCSL Inference}\label{sec:inference}

After training, we derive a model $\pi_\theta(a|\bs, \bg)$ that predicts the next action based on the given state and goal. However, while the goal can be accurately computed through each trajectory in the training data via Eq. (\ref{eq:goal}), we must assign a desirable goal as input for the new state encountered during inference.  Classic GCSL approaches typically determine this goal-choosing strategy based on simple statistics calculated from the training data. E.g., \citet{DecisionTransformer} and \citet{zheng2022online} set the goals for all states at inference as the product of the maximal cumulative reward in training data and a fixed factor serving as a hyperparameter. Similarly, \citet{xin2022rethinking} derive the goals for inference at a given timestep by scaling the mean of the cumulative reward in training data at the same timestep with a pre-defined factor. However, a central yet unexplored question is: \emph{what are the general characteristics of the goals and how can we determine them for inference in a principled manner? } 

In this paper, we investigate the distribution of the multi-dimensional goals that can be achieved during inference by first stating the following theorem. Proof is given in Appendix \ref{sec:proof}.

\begin{theorem}
Assume that the environment is modeled as an MOMDP. Consider a trajectory $\tau$ that is generated by the policy $\pi(a|\bs, \bg)$ given the initial state $\bs_1$ and goal $\bg_1$, the distribution of goals $\bg^a$ that the agent actually achieves throughout the trajectory is determined by $(\bs_1, \bg_1, \pi)$. 
\label{thm}
\end{theorem}

Based on this theorem, we'd like to learn the distribution of $\bg^a$ conditioned on ($\bs_1, \bg_1, \pi)$, denoted as $P(\bg^a|\bs_1, \bg_1, \pi)$. 
General generative models, such as GANs \cite{mirza2014conditional} and diffusion models \cite{ho2022classifier}, can be employed to learn this distribution.
In this paper, we propose the use of a conditional variational auto-encoder (CVAE) \cite{sohn2015learning} due to its simplicity, robustness, and ease of formulation.

% ** ONe sentence to justify use of CVAE ***

\begin{algorithm}[t]\small
 \DontPrintSemicolon
  \textbf{Input:} state $\bs'$, sample size $K$, policy model $\pi$, utility principle $U$, prior $q(\bg'|\bs')$, distribution of achievable goals $P(\bg^a|\bs',\bg',\pi)$ \; 
  \textbf{Intialization:} set of potential input goals $\bG'=\emptyset$, set of expected achievable goals $\bG^a=\emptyset$  \; 
  \For{k = $1, \ldots, K$}{
     Sample a $\bg'_k$ from $q(\bg'|\bs')$ \;
     Compute the expectation of the achievable goal through sampling: 
     $\tilde{\bg}^a_k= \bbE_{\bg^a_k \sim P(\cdot|\bs',\bg'_k,\pi) } \bg^a_k$\;
      $\bG' = \bG' \cup \bg'_k$ \;
      $\bG^a= \bG^a \cup \tilde{\bg}^a_k$ \;
    }
  Choose the best $\tilde{\bg}^a_b$ from $\bG^a$ according to $U(\tilde{\bg}^a)$ \;
  Choose corresponding $\bg'_b$ from $\bG'$ \;
  \textbf{Return:} $\pi(\cdot|\bs',\bg'_b)$
  \caption{Inference of MOGCSL}
\label{alg:inference}
\end{algorithm}
\vspace{-1mm}

Specifically, this distribution can be learned directly on the training data $\cD_{tr}$. For each $(\bs,\bg) \in \cD_{tr}$
% \footnote{\scriptsize{Note that each $(\bs,\bg)$ can be taken as the initial state and goal from the corresponding timestep to the end of the trajectory that it belongs to.}}
, $\bg$ should be a sample from the distribution of the achievable goals by the policy $\pi$, given the initial state $\bs$ and input goal $\bg$.
That's because the policy $\pi$ is trained to imitate the actions demonstrated by each data point in $\cD_{tr}$, where the achieved goal of the trajectory starting from $(\bs,\bg)$ is exactly $\bg$.
Let $\bc=(\bs, \bg, \pi)$.  The loss function is:
\begin{equation}\small
\label{eq:loss_cvae}
\begin{split}
\cL_{CVAE1} = \bbE_{(\bs, \bg) \in \cD_{tr}, z \sim Q_1 } [ -log P_1(\bg|z, \bc)+ D_{KL}(Q_1(z|\bg,\bc)||P(z))], 
\end{split}
\end{equation}

where $Q_1(z|\bg,\bc)$ is the encoder and $P_1(\bg|z,\bc)$ is the decoder. Based on Gaussian distribution assumption, they can be written as $Q_1=\cN(\mu(\bg,\bc), \Sigma(\bg,\bc))$ and $P_1=\cN(f_{CVAE1}(z,\bc), \sigma^2 I)$ respectively, where $z \sim \cN(0, I)$. Then we can derive a sample of $\bg^a$ by inputting a sampled $z$ into $f_{CVAE1}$. We include more discussions of this sampling-based strategy in Appendix \ref{sec:goal_selec}.

On the inference stage, given a new state $\bs'$, we first sample a set of goals $\bg'$ as the possible input of $\pi$ through a learnable prior $q(\bg'|\bs')$. Similarly, we learn this prior via another CVAE on the training data. The loss is:
\begin{equation}\small
\label{eq:loss_vae}
\begin{split}\small
\cL_{CVAE2} = \bbE_{(\bs, \bg) \in \cD_{tr}, z \sim Q_2 }[ -log P_2(\bg|z,\bs)+ D_{KL}(Q_2(z|\bg,\bs)||P(z))].
\end{split}
\end{equation}

Finally, we'll choose a desirable goal as the input along with the new state $\bs'$  encountered in inference by sampling from the two CVAE models.  Specifically, we propose to: (1) sample from the prior $q(\bg'|\bs')$ to get a set of potential input goals, denoted as $\bG'$, (2) for each $\bg' \in \bG'$, estimate the expectation of the actually achievable goal $\tilde{\bg}^a$ by sampling from $P(\cdot|\bs',\bg',\pi)$ and taking the average %\footnote{\tiny{We include a more detailed discussion of this sampling-based strategy in Appendix \ref{sec:goal_selec}.}}
, (3) choose a best goal as input for inference from $\bG'$ according to the associated expected $\tilde{\bg}^a$ by a predefined utility principle $U(\tilde{\bg}^a)$, which generally tends to pick up a ``high'' goal to achieve larger rewards on multiple objectives. The exact definition can be flexible with practical requirements with different objective importance. E.g., choose by a predefined partial ordering.
Note that our algorithm is general and can accommodate adaptations of the utility function, as the desired goals are always set among the set of achievable goals, which naturally resolves the objective-conflicting problem. This definition flexibility allows practitioners to easily tailor MOGCSL to various practical scenarios and objective priorities.
See Algorithm \ref{alg:inference} for inference pseudocode.

\subsection{Analysis of Denoising Capability} \label{sec:theor}
An important benefit of MOGCSL is its capability to remove harmful effects of potentially noisy instances in the training data by leveraging the multiple-objective goals. 
To illustrate this, we consider the following common setup in recommender systems. There is a recommender system that has been operational, and recording data. At each interaction, the system shows the user a list of items. The user then chooses one of these items. In the counterfactual that the recommender system is ideal, the action recorded would be $a$ which reveals the user's true interest. Since the actual recommender system to collect the data is not ideal, we have no direct access to $a$, but rather to a noisy version $\varepsilon(a)$.

We assume that the noisy portion of the training data originates from users who are presented with a list of items that are not suitable for them, rendering their reactions to these recommendations uninformative. Conversely, interactions achieving higher goals are generally less noisy, meaning $\varepsilon(a)$ is closer to $a$.
To illustrate this, consider a scenario where a user clicks two recommended items ($v_1$ and $v_2$) under the same state. After clicking $v_1$, the user chooses to quit the system, while he stays longer and browses more items after clicking $v_2$. This indicates that the goal (i.e., cumulative reward) with $v_2$ is larger than that with $v_1$. In this case, we argue that  $v_2$ should be considered as the user’s truly preferred item over  $v_1$. That's because the act of quitting, which results in a smaller goal, indicates user dissatisfaction with the previous recommendation, even though he did click $v_1$ before.
Our proposed MOGCSL can model and leverage this mechanism based on multi-dimensional goals, which serve as a description of the \textbf{future effects of current actions on multiple objectives}. 
% In other words, MOGCSL tends to prioritize instances that result in high achievable goals on multiple objectives, while discounting those with low goals. 
Specifically, by incorporating multi-dimensional goals as input, MOGCSL can effectively differentiate between noisy and noiseless samples in the training data. During inference, when high goals are specified as input, the model can make predictions based primarily on the patterns learned from the corresponding noiseless interaction data. It's worth noting that several prior studies on recommender system denoising \cite{zhang2025robust, chen2021bias} have highlighted that such noises (i.e., users' noisy instant behaviors that don't align with their long-term, real interests) are very common and widely exist in real-world recommendation datasets.

To empirically validate this effect, we discuss relevant experiments in Appendix \ref{sec:denoise_exp} due to limited space. The results demonstrate the denoising capability of MOGCSL.

%% file: sections/6.experiment.tex
\section{Experiments}

In this section, we introduce our experiments to address the following research questions: 1) \textbf{RQ1.} How does MOGCSL perform when compared to previous methods for multi-objective learning in recommender systems? 2) \textbf{RQ2.} How does MOGCSL mitigate the complexity challenges encountered in prior research? 3) \textbf{RQ3.} How does the goal-generation strategies for inference perform?

\subsection{Experimental Setup}

\textbf{Datasets}  We use two public datasets: Challenge15 and RetailRocket.
%\footnote{https://recsys.acm.org/recsys15/challenge}  \footnote{https://www.kaggle.com/retailrocket/ecommerce-dataset}. 
Both of them include binary labels indicating whether a user clicked or purchased the currently recommended item from real world.  More details of the datasets and implementation specifics are provided in Appendix \ref{sec:expdetail}.

% After preprocessing, the Challenge15 dataset comprises 200,000 sessions, encompassing 26,702 unique items, 1,110,965 clicks and 43,946 purchases. Similarly, the processed RetailRocket dataset consists of 195,523 sessions, involving 70,852 distinct items. It documents 1,176,680 clicks and 57,269 purchases. We partition them into training, validation, and test sets, maintaining an 8:1:1 ratio.

% We conduct experiments on two public datasets: Challenge15 and  RetailRocket. They are both collected from online e-business platforms by recording users' sequential behaviours in recommendation sessions. Specifically, both of them contain binary labels indicating whether the user click or purchase the current recommended item. Following previous research \cite{}, we filter out sessions whose length is shorter than 3 or longer than 50. After processing, Challenge15 contains 200,000 sessions, in which there are 26,702 items, 1,110,965 clicks and 43,946 purchases. The processed RetailRocket contains 195,523 sessions with 70,852 items, 1,176,680 clicks and 57,269 purchases. The datasets are divided into training,validation, and test set with the ratio as 8:1:1.

\textbf{Baselines}  Prior research on multi-objective learning encompass both model structure adaptation and optimization constraints. We consider two representative model architectures: Shared-Bottom \cite{ma2018modeling} and MMOE \cite{ma2018modeling}. For optimization constraints, we compare four methods: Fixed-Weights \cite{wang2016multi} assigns fixed weights for different objectives based on grid search; DWA \cite{liu2019end} dynamically adjusts weights by considering the dynamics of loss change; PE \cite{lin2019pareto} generates Pareto-efficient recommendations across multiple objectives; Nash-MTL \cite{navon2022multi} characterize the Nash bargaining solution for multi-objective learning;
FAMO \cite{liu2024famo} adjusts weights to achieve balanced task loss reduction in an efficient way.

Following previous research\cite{yu2020gradient}, we consider all these optimization methods for each model architecture, resulting in eight baselines denoted as Share-Fix, MMOE-Fix, etc. Additionally, we introduce a variant of PRL \cite{xin2022rethinking}, which firstly applied GCSL to recommender systems. Specifically, similar to classic multi-objective methods, we compute the weighted summation of rewards from different objectives at each timestep. Then the overall cumulative reward is calculated as the goal, which is a scalar following conventional GCSL. We call this variant as MOPRL. Since we formulate our problem as an MOMDP, we also incorporate a recent baseline RMTL \cite{liu2023multi}, which applies offline reinforcement learning for sequential recommendation. Similarly, the aggregate reward is derived by the weighted summation of all objective rewards.
Additionally, we compare against PMORS~\cite{jin2024pareto}, a recent supervised multi-objective learning framework designed for recommender systems.
Note that we employ the \emph{T-enc} and self-attention block introduced in Section \ref{sec:training} as the \textbf{base module to encode sequential data for all compared baselines to ensure fair comparison}.

% To ensure a fair comparison, we employ the \emph{T-enc} and self-attention block introduced in Section \ref{sec:training} as the base module to encode sequential data for all compared baselines.

\textbf{Evaluation Metrics} We employ two widely recognized information retrieval metrics to evaluate model performance in top-$k$ recommendation: Hit Ratio (HR@$k$) and  Normalized discounted cumulative gain (NG@$k$).We report the mean and standard deviation over 5 seeds for each experiment.

\input{sections/tables_1}

% \vspaceMOGCSL

\subsection{Performance Comparison (RQ1)}\label{sec:performance}

We begin by comparing the performance of MOGCSL and selected baselines in terms of top-$k$ recommendation, as reported in Table \ref{tab:main_1}. It's worth mentioning that a straightforward strategy based on training set statistics is employed to determine the inference goals in PRL \cite{xin2022rethinking}. Specifically, at each timestep in inference, the goals are set as the average cumulative reward from offline data at the same timestep, multiplied by a hyper-parameter factor $\lambda$ that is tuned using the validation set. To ensure a fair comparison, we adopt the same strategy here for determining inference goals in MOGCSL. The comparison between different goal-choosing strategies is discussed in Section \ref{sec:goal-generation}.

On RetailRocket, MOGCSL significantly outperforms previous multi-objective methods on purchase. Regarding click metrics, MOGCSL achieves the best performance on HR, while Share-PE slightly outperforms it on NDCG. However, the performance gap between Share-PE and MOGCSL for purchase-related metrics ranges from 17\% to 20\%, whereas Share-PE only marginally outperforms MOGCSL on NDCG for purchase by less than 1\%. 
%Note that PE aims to pursue the Pareto-efficient state, which indicates that improving one objective cannot come at the cost of decrease in performance of any other objectives. However, this assumption could be unreasonable for recommender systems in many cases. As shown by the experiment results, a substantial increase in purchase with only a negligible decrease in click is preferred here, which can be achieved by MOGCSL with no prior assignments of optimization constraints.
 Additionally, we observe that the more complex architecture design, MMOE, can perform worse than the simpler Shared-Bottom in many cases. Surprisingly, a naive optimization strategy based on fixed loss weights can outperform more advanced methods like DWA across several metrics. These findings highlight the limitations of previous approaches that rely on assumptions about model architectures or optimization constraints, which may not necessarily be true in general environments. Similar trends are observed on Challenge15. While MMOE-PE performs slightly better on click by 1-2\%, MOGCSL achieves a substantial performance improvement on the more important purchase metrics by 11-20\%.
{ Especially, PE optimizes for Pareto efficiency by assuming that improving one objective should not degrade performance on any other objective. However, this assumption does not align well with the realities of recommender systems. In practice, a significant improvement in high-value objectives (e.g., purchase) accompanied by only minor reductions in less critical objectives (e.g., click) is often preferable in practical scenarios. } We've included more discussions on the relative performance in Appendix \ref{sec:rela} and \ref{sec:eff_goal}.

Apart from classic benchmarks for multi-objective learning,  MOGCSL also exhibits significant and consistent performance improvements on both datasets compared to offline RL based RMTL and the variant MOPRL created on standard GCSL. Especially, at each timestep, the overall reward of them is calculated as the weighted sum of rewards across all objectives. In our experiments, it's defined as $r'=w_c r^c + w_p r^p$, where $r^c$ and $r^p$ are click and purchase reward and $w_c + w_p = 1$. Then the goal in MOPRL is derived by calculating the cumulative rewards as a scalar. In contrast, MOGCSL takes the goal as a vector, allowing the disentanglement of rewards for different objectives along different dimensions. Notably, no additional summation weights or other constraints are required.
We have conducted experiments to compare the performance of MOGCSL to MOPRLs with different weight combinations. 
The results shown in Appendix \ref{sec:moprls} demonstrate that MOGCSL consistently outperforms MOPRL across all weight combinations, demonstrating that representing the goal as a multi-dimensional vector enhances the effectiveness of GCSL on multi-objective learning. 
% See the figure for the comparison in Appendix \ref{sec:moprls}. 

\subsection{Complexity Comparison (RQ2)}

Apart from the performance improvement, MOGCSL also benefits from seamless integration 
\begin{wraptable}{rh}{0.45\textwidth}
\renewcommand\arraystretch{1.0}
\footnotesize
\centering
\caption{Time/Space complexity comparison. }
% \vspace{-5px}
\label{tab:complexity}
\begin{tabular}{ccc}
\hline
&Model Size&Training time \\\hline
Share-Fix&14.0M&9.6Ks\\
Share-DWA&14.0M&5.3Ks\\
Share-PE&14.0M&5.6Ks\\
Share-Nash&14.0M&8.6Ks\\
Share-FAMO&14.0M&5.1Ks\\
MMOE-Fix&14.1M&10.2Ks\\
MMOE-DWA&14.1M&9.5Ks\\
MMOE-PE&14.1M&60.5Ks\\
MMOE-Nash&14.1M&12.4Ks\\
MMOE-FAMO&14.1M&8.8Ks\\
RMTL&17.5M&100.2Ks\\
PMORS&14.2M&10.4Ks\\
MOPRL&9.1M&3.2Ks\\
\hline
MOGCSL&9.1M&3.0Ks\\
\hline
\end{tabular}
\vspace{-2mm}
\end{wraptable}
with classic sequential models, adding minimal additional complexity. During the training stage, the only extra complexity
arises from relabeling one-step rewards with goals and including them as  input to  the sequential model. In contrast, previous  multi-objective learning methods often introduce significantly excess time and space complexity \cite{zhang2021survey}.
For instance, MMOE and Shared-Bottom both design separate towers for each objective \cite{ma2018modeling}, 
leading to a significant increase in model 
parameters as the number of tasks grows;
RMTL needs an additional RL head for optimization on TD error. In contrast, MOGCSL only requires  a simple MLP layer for action projection.
In terms of time complexity, DWA, Nash-MTL and FAMO require recording and calculating loss change dynamics for each training epoch, while PE and PMORS involve computing the inverse of a large parameter matrix to solve an optimization problem under KKT conditions. Additionally, tuning the weight combinations for multiple objectives using grid search in methods like Fix-Weight and MOPRL is highly time-consuming, requiring approximately $O(m^{n-1})$ repetitive experiments to identify a near-optimal combination, where $m$ is the size of the search space per dimension and $n$ is the number of objectives. In contrast, MOGCSL inherently avoids this weight-tuning process. Table \ref{tab:complexity} summarizes the comparison on RetailRocket, where MOGCSL significantly benefits from a smaller model size and faster training speed. See more discussions of complexity in Appendix \ref{sec:comput}.
% , while concurrently achieving great performance.

% Apart from the performance improvement, MOGCSL also benefits from the feasibility of being seamlessly integrated with classic sequential model with minimal additional complexity. More specifically, during the training stage, the only extra complexity comes from relabeling one-step rewards with goals and including them as the input of the sequential model. In contrast, previous methods on multi-objective learning usually bring in significant time and space excess complexity. For example, MMOE and Shared-Bottom both design separate towers for each objective, resulting in large increase in model parameters. DWA needs to record and calculate the loss change dynamics of each training epoch, while PE requires calculation of the inverse of a large parameter matrix for resolving an optimization problem under a KKT condition. For Fix-Weight and MOPRL, tuning the combination of weights for multiple objectives by grid search is tedious and time-consuming. $O(m^d)$ experiments are needed to find the nearly optimal combination, where $m$ is the size of search space of each weight and $d$ is the number of objectives.

\subsection{Goal-generation Strategy Comparison (RQ3)}\label{sec:goal-generation}

As introduced in Section \ref{sec:inference}, most previous research on GCSL decides the inference goals based on simple statistics on the training set. However, we demonstrate that the distribution of the goals achieved by the agent during inference should be jointly determined by the initial state, input goal and behavior policy. Based on that, we propose a novel algorithm (see Algorithm \ref{alg:inference}) that leverages CVAE to derive feasible and desirable goals for inference. Note that an utility principle $U(\bg)$ is required to evaluate the goodness of the multi-dimensional goals, which is generally preferable for ``high'' goals but could be flexible with specific business requirements. In our experiments, we select the best goal $\tilde{\bg}^a_b$ from the set $\bG^a$ based on the following rule, which ensures that no goal within the achievable set is superior to the selected goal across all objectives:
% \footnote{\scriptsize{It ensures that no goal within the achievable set is superior to the selected goal across all objectives.}}:
% If there are multiple samples satisfying this condition, we randomly select one as the best goal.}}:
\begin{equation}\small
\label{eq:U_g}
\tilde{\bg}^a_b = \tilde{\bg} \in \bG^a, \;
\mathrm {s.t.} \ \nexists \, \tilde{\bg}' \in \bG^a \setminus \tilde{\bg}\,,\ \tilde{g}'_i \geq \tilde{g}_i \ \forall i \in [1, d].
\end{equation}

We compare two variants of MOGCSL here. MOGCSL-S employs the statistical strategy introduced in Section \ref{sec:performance}, while MOGCSL-C utilizes the goal-choosing method based on CVAE. As shown in Table~\ref{tab:main_3}, we observe that these two strategies do not significantly differ in overall performance across both datasets. While MOGCSL-C performs better on RetailRocket, it exhibits worse performance on Challenge15. To investigate the reason, we conduct an additional experiment by varying the factor $\lambda$ for the inference goals of MOGCSL-S. The results reveals that the optimal performance is achieved when the factor is set between 1 and 2 for all metrics. When it grows larger, performance consistently declines. The figure is shown in Appendix \ref{sec:scalar_comp}. Similar findings have been reported in prior research \cite{DecisionTransformer, xin2022rethinking, zheng2022online}, demonstrating that setting very large inference goals can indeed harm performance.

We posit that the sparsity of training data within the high-goal space may contribute to the suboptimal performance of more advanced goal-choosing methods. While we may find some potentially achievable high goals, the model lacks sufficient training data to learn effective actions to reach these goals. Notably, the mean cumulative reward across all trajectories in both datasets is only around 5.3 for click and 0.2 for purchase. Consequently, most training data demonstrates how to achieve relatively low goals, hindering the model’s ability to generalize effectively for larger goals in inference. To validate this, we conduct additional experiments on another dataset, Tenrec, with higher average goals. The two goals are defined for click and like behaviors. The mean cumulative rewards across all trajectories are approximately 28.3 for click and 1.2 for like, both notably higher than those observed in the RetailRocket and Challenge15 datasets. The results in Table \ref{tab:main_3} show that MOGCSL-C significantly outperforms MOGCSL-S on Tenrec.  Together with previous experiments, it shows that insufficient high-quality trajectories in offline datasets can lead to suboptimal outcomes when GCSL policies are overly optimistic. 

It's worth noting that a very recent work, RVDT~\cite{bairebalancing}, provides an initial attempt at offering a more formal and theoretically grounded analysis of this problem. Leveraging insights and relevant theory from offline reinforcement learning, the authors draw similar conclusions and demonstrate that insufficient high-quality trajectories in offline datasets can lead to suboptimal outcomes when GCSL policies are overly optimistic.

\input{sections/tables_2}

Applying these insights to our setting, although the CVAE-based algorithm can identify potentially achievable high-value goals, the GCSL model may lack sufficient training data to learn effective actions for reaching these goals. Consequently, a more balanced strategy, such as selecting goals based on the average cumulative goals in the offline data, scaled by a hyperparameter factor, often proves more effective when high-goal training data is limited.

Importantly, we'd like to highlight that this {doesn't weaken the framework's practical value}. MOGCSL-S already outperforms existing baselines with low inference cost, while MOGCSL-C provides further gains when richer high-goal data is available. We therefore view the two variants as complementary: MOGCSL-S is a strong default choice, whereas MOGCSL-C is preferable in less sparse, higher-goal regimes. Thus, rather than limiting the practicality, this observation helps for appropriate strategy selection tailored to practical use cases.
See more discussion in Appendix \ref{sec:general}.

The results provide several valuable insights for selecting goal-choosing strategies for MOGCSL. First, strategies based on simple statistics on the training data prove to be efficient and effective in most cases, particularly when low computation complexity is required during inference.
%, MOGCSL with a statistical goal-choosing strategy stands out as a robust and strong method for multi-objective recommendation. 
Second, if we aim to further enhance performance using more advanced goal-choosing algorithms, access to a training set with more instances with high-valued goals could be crucial. 
Last, our  MOGCSL proposes to address and model the achievability of multi-objective goals at inference. Within this framework, we explore and validate {both principled (MOGCSL-C) and heuristic (MOGCSL-S) goal-selection methodologies} to derive these achievable and ``high'' goals on multiple objectives. These methods demonstrate complementary strengths across diverse datasets and empirical conditions, which is itself a non-trivial finding that is valuable for practical application of MOGCSL.

%% file: sections/tables_1.tex
\begin{table*}[t]\footnotesize
    \centering
    \begin{threeparttable}
    \caption{Comparison between MOGCSL and other baselines on RetailRocket and Challenge15. }
    % The mean and standard deviation over 5 seeds are reported. Boldface denotes the best results.}
    %: RCF outperforms all baselines on all metrics with significance level $p$-value < 0.05 (indicated by $$). The last row shows the $p$-value of comparing RCF with the best baseline on the corresponding metric (indicated by boldface).}\vspace{-5pt}
    \label{tab:main_1}
        \begin{tabular}{p{1.8cm}p{0.53cm}<{\centering}p{0.53cm}<{\centering}p{0.53cm}<{\centering}p{0.53cm}<{\centering}p{0.53cm}<{\centering}p{0.53cm}<{\centering}p{0.53cm}<{\centering}p{0.53cm}<{\centering}}
    \toprule
    \multirow{2}{*}{[RetailRocket]}&\multicolumn{4}{c}{Purchase (\%)}&\multicolumn{4}{c}{Click (\%)}\cr
    \cmidrule(lr){2-5} \cmidrule(lr){6-9}
    &\multicolumn{1}{c}{HR@10\scriptsize}&\multicolumn{1}{c}{NG@10\scriptsize}&\multicolumn{1}{c}{HR@20\scriptsize}&\multicolumn{1}{c}{NG@20\scriptsize}&\multicolumn{1}{c}{HR@10\scriptsize}&\multicolumn{1}{c}{NG@10\scriptsize}&\multicolumn{1}{c}{HR@20}&\multicolumn{1}{c}{NG@20\scriptsize}\cr
    \midrule
    { Share-Fix} & 48.57\tiny{$\pm$0.17} & 45.79\tiny{$\pm$0.09} & 49.47\tiny{$\pm$0.10} & 46.01\tiny{$\pm$0.08} & 35.51\tiny{$\pm$0.24} & 25.85\tiny{$\pm$0.16} & 40.15\tiny{$\pm$0.20} & 27.03\tiny{$\pm$0.14} \\
    Share-DWA & 48.11\tiny{$\pm$0.04} & 45.83\tiny{$\pm$0.08} & 48.64\tiny{$\pm$0.08} & 45.96\tiny{$\pm$0.06} & 34.20\tiny{$\pm$0.27} & 25.53\tiny{$\pm$0.12} & 38.43\tiny{$\pm$0.31} & 26.60\tiny{$\pm$0.13} \\
    Share-PE & 48.67\tiny{$\pm$0.12} & 45.94\tiny{$\pm$0.03} & 49.42\tiny{$\pm$0.04} & 46.13\tiny{$\pm$0.02} & 35.69\tiny{$\pm$0.08} & \textbf{26.20\tiny{$\pm$0.08}} & 40.28\tiny{$\pm$0.11} & \textbf{27.37\tiny{$\pm$0.07}} \\
    Share-Nash & 48.84\tiny{$\pm$0.21} & 46.23\tiny{$\pm$0.17} & 50.07\tiny{$\pm$0.19} & 46.50\tiny{$\pm$0.22} & 35.87\tiny{$\pm$0.20} & {26.04\tiny{$\pm$0.16}} & 40.52\tiny{$\pm$0.25} & 27.11\tiny{$\pm$0.13} \\
    Share-FAMO & 48.92\tiny{$\pm$0.17} & 46.11\tiny{$\pm$0.10} & 50.19\tiny{$\pm$0.09} & 46.78\tiny{$\pm$0.13} & 35.97\tiny{$\pm$0.14} & {26.17\tiny{$\pm$0.09}} & 40.71\tiny{$\pm$0.11} & 27.35\tiny{$\pm$0.18} \\
    MMOE-Fix & 47.74\tiny{$\pm$0.09} & 44.01\tiny{$\pm$0.05} & 48.61\tiny{$\pm$0.11} & 44.23\tiny{$\pm$0.04} & 35.29\tiny{$\pm$0.16} & 25.67\tiny{$\pm$0.09} & 40.04\tiny{$\pm$0.26} & 26.87\tiny{$\pm$0.11} \\
    MMOE-DWA & 47.78\tiny{$\pm$0.40} & 44.57\tiny{$\pm$0.13} & 48.44\tiny{$\pm$0.24} & 44.79\tiny{$\pm$0.09} & 35.68\tiny{$\pm$0.46} & 26.13\tiny{$\pm$0.29} & 40.22\tiny{$\pm$0.57} & 27.28\tiny{$\pm$0.32} \\
    MMOE-PE & 46.58\tiny{$\pm$0.22} & 43.72\tiny{$\pm$0.15} & 47.37\tiny{$\pm$0.11} & 43.94\tiny{$\pm$0.14} & 35.39\tiny{$\pm$0.27} & 26.19\tiny{$\pm$0.11} & 39.78\tiny{$\pm$0.39} & 27.31\tiny{$\pm$0.14} \\
    MMOE-Nash & 47.88\tiny{$\pm$0.27} & 46.49\tiny{$\pm$0.29} & 51.17\tiny{$\pm$0.22} & 47.10\tiny{$\pm$0.31} & 35.79\tiny{$\pm$0.18} & 26.02\tiny{$\pm$0.27} & 41.03\tiny{$\pm$0.20} & 26.39\tiny{$\pm$0.28} \\
    MMOE-FAMO & 47.93\tiny{$\pm$0.32} & 46.42\tiny{$\pm$0.23} & 51.24\tiny{$\pm$0.19} & 47.15\tiny{$\pm$0.21} & 35.92\tiny{$\pm$0.21} & 26.14\tiny{$\pm$0.17} & 41.15\tiny{$\pm$0.39} & 26.62\tiny{$\pm$0.14} \\
    RMTL & 62.84\tiny{$\pm$0.24} & 49.92\tiny{$\pm$0.20} & 67.04\tiny{$\pm$0.15} & 50.89\tiny{$\pm$0.25} & 33.95\tiny{$\pm$0.35} & 24.23\tiny{$\pm$0.41} & 39.87\tiny{$\pm$0.32} & 25.27\tiny{$\pm$0.29} \\
    PMORS & 63.14\tiny{$\pm$0.15} & 51.02\tiny{$\pm$0.17} & 67.45\tiny{$\pm$0.20} & 52.07\tiny{$\pm$0.13} & 34.16\tiny{$\pm$0.26} & 24.09\tiny{$\pm$0.23} & 39.86\tiny{$\pm$0.19} & 25.31\tiny{$\pm$0.27} \\
    MOPRL & 61.18\tiny{$\pm$0.19} & 50.74\tiny{$\pm$0.10} & 64.76\tiny{$\pm$0.25} & 51.65\tiny{$\pm$0.02} & 33.99\tiny{$\pm$0.11} & 24.31\tiny{$\pm$0.08} & 38.98\tiny{$\pm$0.08} & 25.57\tiny{$\pm$0.08} \\
    \hline
    MOGCSL & \textbf{65.43\tiny{$\pm$0.15} }& \textbf{52.92\tiny{$\pm$0.11} }& \textbf{69.28\tiny{$\pm$0.14}} & \textbf{53.90\tiny{$\pm$0.14}} & \textbf{36.30\tiny{$\pm$0.25}} & 25.24\tiny{$\pm$0.15} & \textbf{41.92\tiny{$\pm$0.35}} & 26.67\tiny{$\pm$0.19} \\
    
    \bottomrule
    \end{tabular}

    \vspace{0.2cm}

    \begin{tabular}{p{1.8cm}p{0.53cm}<{\centering}p{0.53cm}<{\centering}p{0.53cm}<{\centering}p{0.53cm}<{\centering}p{0.53cm}<{\centering}p{0.53cm}<{\centering}p{0.53cm}<{\centering}p{0.53cm}<{\centering}}
    \toprule
    \multirow{2}{*}{[Challenge15]}&\multicolumn{4}{c}{Purchase (\%)}&\multicolumn{4}{c}{Click (\%)}\cr
    \cmidrule(lr){2-5} \cmidrule(lr){6-9}
    &\multicolumn{1}{c}{HR@10}&\multicolumn{1}{c}{NG@10}&\multicolumn{1}{c}{HR@20}&\multicolumn{1}{c}{NG@20}&\multicolumn{1}{c}{HR@10}&\multicolumn{1}{c}{NG@10}&\multicolumn{1}{c}{HR@20}&\multicolumn{1}{c}{NG@20}\cr
    \midrule
    Share-Fix & 38.18\tiny{$\pm$0.10} & 25.47\tiny{$\pm$0.26} & 43.97\tiny{$\pm$0.18} & 26.93\tiny{$\pm$0.20} & 41.61\tiny{$\pm$0.30} & 25.77\tiny{$\pm$0.16} & 49.19\tiny{$\pm$0.43} & 27.70\tiny{$\pm$0.19} \\
    Share-DWA & 38.27\tiny{$\pm$0.18} & 25.63\tiny{$\pm$0.08} & 43.95\tiny{$\pm$0.19} & 27.07\tiny{$\pm$0.08} & 41.49\tiny{$\pm$0.24} & 25.90\tiny{$\pm$0.16} & 48.90\tiny{$\pm$0.20} & 27.77\tiny{$\pm$0.14} \\
    Share-PE & 38.92\tiny{$\pm$0.09} & 25.83\tiny{$\pm$0.13} & 44.82\tiny{$\pm$0.12} & 27.32\tiny{$\pm$0.07} & 42.46\tiny{$\pm$0.16} & 26.39\tiny{$\pm$0.06} & 50.05\tiny{$\pm$0.17} & 28.32\tiny{$\pm$0.06} \\
    Share-Nash & 39.12\tiny{$\pm$0.17} & 25.87\tiny{$\pm$0.15} & 45.13\tiny{$\pm$0.22} & 27.56\tiny{$\pm$0.14} & 43.10\tiny{$\pm$0.24} & 26.89\tiny{$\pm$0.19} & 50.61\tiny{$\pm$0.17} & 28.88\tiny{$\pm$0.21} \\
    Share-FAMO & 39.06\tiny{$\pm$0.13} & 25.94\tiny{$\pm$0.21} & 44.97\tiny{$\pm$0.18} & 27.65\tiny{$\pm$0.09} & 43.25\tiny{$\pm$0.13} & 27.02\tiny{$\pm$0.11} & 50.58\tiny{$\pm$0.26} & 28.94\tiny{$\pm$0.14} \\
    MMOE-Fix & 35.34\tiny{$\pm$0.12} & 23.87\tiny{$\pm$0.07} & 40.68\tiny{$\pm$0.09} & 25.22\tiny{$\pm$0.12} & 43.82\tiny{$\pm$0.16} & 27.33\tiny{$\pm$0.09} & 51.42\tiny{$\pm$0.19} & 29.26\tiny{$\pm$0.10} \\
    MMOE-DWA & 37.04\tiny{$\pm$0.40} & 24.88\tiny{$\pm$0.13} & 42.64\tiny{$\pm$0.24} & 26.30\tiny{$\pm$0.09} & 42.20\tiny{$\pm$0.46} & 26.45\tiny{$\pm$0.29} & 49.48\tiny{$\pm$0.57} & 28.30\tiny{$\pm$0.32} \\
    MMOE-PE & 36.40\tiny{$\pm$0.36} & 24.66\tiny{$\pm$0.19} & 41.52\tiny{$\pm$0.33} & 25.96\tiny{$\pm$0.19} & \textbf{44.04\tiny{$\pm$0.09}} & \textbf{27.44\tiny{$\pm$0.03}} & \textbf{51.60\tiny{$\pm$0.07}} & \textbf{29.37\tiny{$\pm$0.03}} \\
    MMOE-Nash & 38.01\tiny{$\pm$0.41} & 25.40\tiny{$\pm$0.32} & 43.69\tiny{$\pm$0.27} & 27.06\tiny{$\pm$0.29} & 43.77\tiny{$\pm$0.41} & 26.87\tiny{$\pm$0.52} & 50.96\tiny{$\pm$0.59} & 28.92\tiny{$\pm$0.50} \\ 
    MMOE-FAMO & 37.92\tiny{$\pm$0.56} & 25.43\tiny{$\pm$0.22} & 43.63\tiny{$\pm$0.29} & 27.12\tiny{$\pm$0.15} & 43.71\tiny{$\pm$0.55} & 26.98\tiny{$\pm$0.38} & 50.98\tiny{$\pm$0.66} & 29.03\tiny{$\pm$0.39} \\
    RMTL & 53.34\tiny{$\pm$0.46} & 33.65\tiny{$\pm$0.36} & 62.89\tiny{$\pm$0.49} & 34.74\tiny{$\pm$0.37} & 41.41\tiny{$\pm$0.41} & 24.87\tiny{$\pm$0.38} & 49.29\tiny{$\pm$0.43} & 26.91\tiny{$\pm$0.29} \\
    PMORS & 54.98\tiny{$\pm$0.31} & 34.77\tiny{$\pm$0.28} & 63.52\tiny{$\pm$0.24} & 37.05\tiny{$\pm$0.35} & 42.36\tiny{$\pm$0.39} & 25.10\tiny{$\pm$0.42} & 50.07\tiny{$\pm$0.30} & 26.84\tiny{$\pm$0.26} \\
    MOPRL & 54.79\tiny{$\pm$0.37} & 35.37\tiny{$\pm$0.26} & 63.10\tiny{$\pm$0.45} & 37.49\tiny{$\pm$0.27} & 42.14\tiny{$\pm$0.21} & 25.62\tiny{$\pm$0.18} & 50.18\tiny{$\pm$0.25} & 27.66\tiny{$\pm$0.19} \\
    \hline
    MOGCSL & \textbf{56.82\tiny{$\pm$0.25}} & \textbf{35.93\tiny{$\pm$0.15}} & \textbf{65.64\tiny{$\pm$0.55}} & \textbf{38.17\tiny{$\pm$0.19}} & {42.47}\tiny{$\pm$0.15} & {25.64}\tiny{$\pm$0.11} & {50.52}\tiny{$\pm$0.14} & {27.73}\tiny{$\pm$0.11}  \\
    \bottomrule
    \end{tabular}

    \vspace{-3mm}
    \end{threeparttable}
\end{table*}

%% file: sections/tables_2.tex
\begin{table*}[t]\small
    \centering
    \begin{threeparttable}
    \caption{Comparison between statistical strategy and CVAE-based method for goal-choosing.}
    %  on RetailRocket and Challenge15. The mean and standard deviation over 5 seeds are reported. Boldface denotes the best results.}
    %: RCF outperforms all baselines on all metrics with significance level $p$-value < 0.05 (indicated by $$). The last row shows the $p$-value of comparing RCF with the best baseline on the corresponding metric (indicated by boldface).}\vspace{-5pt}
    \label{tab:main_3}
     \begin{tabular}{p{1.8cm}p{0.65cm}<{\centering}p{0.65cm}<{\centering}p{0.65cm}<{\centering}p{0.65cm}<{\centering}p{0.65cm}<{\centering}p{0.65cm}<{\centering}p{0.65cm}<{\centering}p{0.65cm}<{\centering}}
    \toprule
    \multirow{2}{*}{[RetailRocket]}&\multicolumn{4}{c}{Purchase (\%)}&\multicolumn{4}{c}{Click (\%)}\cr
    \cmidrule(lr){2-5} \cmidrule(lr){6-9}
    &\multicolumn{1}{c}{HR@10}&\multicolumn{1}{c}{NG@10}&\multicolumn{1}{c}{HR@20}&\multicolumn{1}{c}{NG@20}&\multicolumn{1}{c}{HR@10}&\multicolumn{1}{c}{NG@10}&\multicolumn{1}{c}{HR@20}&\multicolumn{1}{c}{NG@20}\cr
    \midrule
    MOGCSL-S & \textbf{65.43\tiny{$\pm$0.15} }& \textbf{52.92\tiny{$\pm$0.11} }& {69.28\tiny{$\pm$0.14}} & {53.90\tiny{$\pm$0.14}} & {36.30\tiny{$\pm$0.25}} & 25.24\tiny{$\pm$0.15} & {41.92\tiny{$\pm$0.35}} & 26.67\tiny{$\pm$0.19} \\
    MOGCSL-C & {65.01\tiny{$\pm$0.07} }& {52.89\tiny{$\pm$0.04} }& \textbf{69.34\tiny{$\pm$0.05}} & \textbf{54.00\tiny{$\pm$0.04}} & \textbf{36.54\tiny{$\pm$0.02}} & \textbf{25.41\tiny{$\pm$0.04}} & \textbf{42.20\tiny{$\pm$0.03}} & \textbf{26.84\tiny{$\pm$0.06}} \\

    \bottomrule
    \end{tabular}

    \vspace{0.2cm}

    \begin{tabular}{p{1.8cm}p{0.65cm}<{\centering}p{0.65cm}<{\centering}p{0.65cm}<{\centering}p{0.65cm}<{\centering}p{0.65cm}<{\centering}p{0.65cm}<{\centering}p{0.65cm}<{\centering}p{0.65cm}<{\centering}}
    \toprule
    \multirow{2}{*}{[Challenge15]}&\multicolumn{4}{c}{Purchase (\%)}&\multicolumn{4}{c}{Click (\%)}\cr
    \cmidrule(lr){2-5} \cmidrule(lr){6-9}
    &\multicolumn{1}{c}{HR@10}&\multicolumn{1}{c}{NG@10}&\multicolumn{1}{c}{HR@20}&\multicolumn{1}{c}{NG@20}&\multicolumn{1}{c}{HR@10}&\multicolumn{1}{c}{NG@10}&\multicolumn{1}{c}{HR@20}&\multicolumn{1}{c}{NG@20}\cr
    \midrule
    MOGCSL-S & \textbf{56.82\tiny{$\pm$0.25}} & \textbf{35.93\tiny{$\pm$0.15}} & \textbf{65.64\tiny{$\pm$0.55}} & \textbf{38.17\tiny{$\pm$0.19}} & \textbf{42.47\tiny{$\pm$0.15}} & \textbf{{25.64}\tiny{$\pm$0.11}} & \textbf{{50.52}\tiny{$\pm$0.14}} & \textbf{{27.73}\tiny{$\pm$0.11}}  \\
    MOGCSL-C & {55.13\tiny{$\pm$0.07} }& {35.04\tiny{$\pm$0.02} }& {63.98\tiny{$\pm$0.04}} & {37.30\tiny{$\pm$0.03}} & {42.14\tiny{$\pm$0.04}} & 25.37\tiny{$\pm$0.07} & {50.12\tiny{$\pm$0.09}} & 27.53\tiny{$\pm$0.05} \\

    \bottomrule
    \end{tabular}

    \vspace{0.2cm}

    \begin{tabular}{p{1.8cm}p{0.65cm}<{\centering}p{0.65cm}<{\centering}p{0.65cm}<{\centering}p{0.65cm}<{\centering}p{0.65cm}<{\centering}p{0.65cm}<{\centering}p{0.65cm}<{\centering}p{0.65cm}<{\centering}}
    \toprule
    \multirow{2}{*}{[Tenrec]}&\multicolumn{4}{c}{Like (\%)}&\multicolumn{4}{c}{Click (\%)}\cr
    \cmidrule(lr){2-5} \cmidrule(lr){6-9}
    &\multicolumn{1}{c}{HR@10}&\multicolumn{1}{c}{NG@10}&\multicolumn{1}{c}{HR@20}&\multicolumn{1}{c}{NG@20}&\multicolumn{1}{c}{HR@10}&\multicolumn{1}{c}{NG@10}&\multicolumn{1}{c}{HR@20}&\multicolumn{1}{c}{NG@20}\cr
    \midrule
    MOGCSL-S & {5.96\tiny{$\pm$0.17}} & {2.15\tiny{$\pm$0.12}} & {6.93\tiny{$\pm$0.20}} & {2.70\tiny{$\pm$0.11}} & {4.87\tiny{$\pm$0.13}} & {{1.52}\tiny{$\pm$0.08}} & {{5.67}\tiny{$\pm$0.14}} & {{1.95}\tiny{$\pm$0.11}}  \\
    MOGCSL-C & \textbf{6.78\tiny{$\pm$0.07} }& \textbf{2.99\tiny{$\pm$0.02} }& \textbf{7.84\tiny{$\pm$0.05}} & \textbf{3.86\tiny{$\pm$0.04}} & \textbf{5.66\tiny{$\pm$0.05}} & \textbf{2.14\tiny{$\pm$0.05}} & \textbf{6.73\tiny{$\pm$0.03}} & \textbf{2.68\tiny{$\pm$0.07}} \\

    \bottomrule
    \end{tabular}
    \vspace{-2mm}

    \end{threeparttable}
\end{table*}

%% file: sections/appen_proof.tex
\section{Proof of Theorem 1}\label{sec:proof}

\begin{proof}
Let $H := |\tau|$ be the length of the trajectory, and write
\[
\tau = (\bs_1,a_1,r_1,\ldots,\bs_H,a_H,r_H,\bs_{H+1}).
\]
Fix the MOMDP environment, the initial state $s_1$, the initial goal $g_1$,
and the goal-conditioned policy $\pi(a \mid \bs,\bg)$.

For each timestep $t$, the goal input to the policy is updated according to
\[
\bg_{t+1} = \bg_t - r_t,
\]
or equivalently,
\[
\bg_t = \bg_1 - \sum_{i=1}^{t-1} r_i,
\qquad t=1,\ldots,H,
\]
where the empty sum is defined as $0$. Hence, $\bg_t$ is a deterministic
function of $(\bg_1,r_1,\ldots,r_{t-1})$.

By the chain rule, the conditional distribution of the full trajectory is
\[
\Pr(\tau \mid \bs_1,\bg_1,\pi)
=
\prod_{t=1}^{H}
\Pr(a_t,r_t,\bs_{t+1} \mid \bs_{1:t},a_{1:t-1},r_{1:t-1},\bg_1,\pi).
\]
Since the policy chooses actions according to the current state and current
goal, and since the environment is an MOMDP with transition probability
$T(\bs' \mid \bs,a)$ and reward distribution $P(r \mid \bs,a)$, we have
\[
\Pr(a_t,r_t,\bs_{t+1} \mid \bs_{1:t},a_{1:t-1},r_{1:t-1},\bg_1,\pi)
=
\pi(a_t \mid \bs_t,\bg_t)\, P(r_t \mid \bs_t,a_t)\, T(\bs_{t+1} \mid \bs_t,a_t).
\]
Therefore,
\[
\Pr(\tau \mid \bs_1,\bg_1,\pi)
=
\prod_{t=1}^{H}
\pi(a_t \mid \bs_t,\bg_t)\, P(r_t \mid \bs_t,a_t)\, T(\bs_{t+1} \mid \bs_t,a_t).
\]
Substituting
\[
\bg_t = \bg_1 - \sum_{i=1}^{t-1} r_i
\]
gives
\[
\Pr(\tau \mid \bs_1,\bg_1,\pi)
=
\prod_{t=1}^{H}
\pi\!\left(a_t \,\middle|\, \bs_t,\, \bg_1-\sum_{i=1}^{t-1} r_i\right)
P(r_t \mid \bs_t,a_t)\,
T(\bs_{t+1} \mid \bs_t,a_t).
\]
Because the environment is fixed, both $T(\cdot \mid \bs,a)$ and
$P(\cdot \mid \bs,a)$ are fixed. Hence, the distribution of the full
trajectory $\tau$ is completely determined by $(\bs_1,\bg_1,\pi)$.

Now define
\[
\bg^a(\tau) := \sum_{t=1}^H r_t.
\]
Since $\bg^a(\tau)$ is a deterministic function of the trajectory $\tau$, for any
measurable set $B \subseteq \mathbb{R}^d$, we have

\[
\Pr(\bg^a \in B \mid \bs_1,\bg_1,\pi)
=
\int \mathbf{1}\!\left\{\sum_{t=1}^{H} r_t \in B\right\}
\, \Pr(d\tau \mid \bs_1,\bg_1,\pi).
\]

Therefore, the distribution of $\bg^a$ is completely determined by the
distribution of $\tau$. Since we have shown that $\Pr(\tau \mid \bs_1,\bg_1,\pi)$
is determined by $(\bs_1,\bg_1,\pi)$, it follows that the distribution of $\bg^a$ is
also determined by $(\bs_1,\bg_1,\pi)$.

This proves Theorem 1.
\end{proof}

\paragraph{Remark.}
If the session length is random, the same argument still applies by introducing
an absorbing terminal state with zero reward after termination, or equivalently
by integrating over all possible finite trajectory lengths.

%% file: sections/appen_expdetail.tex
\section{Experiment Details}\label{sec:expdetail}

\subsection{Model Structure}

The model structure of MOGCSL is shown in Figure \ref{fig:structure}.

\begin{figure}[htbp]
\vspace{-1mm}
\centering
\begin{minipage}{0.7\linewidth}
    \centering
    \includegraphics[width=0.9\linewidth]{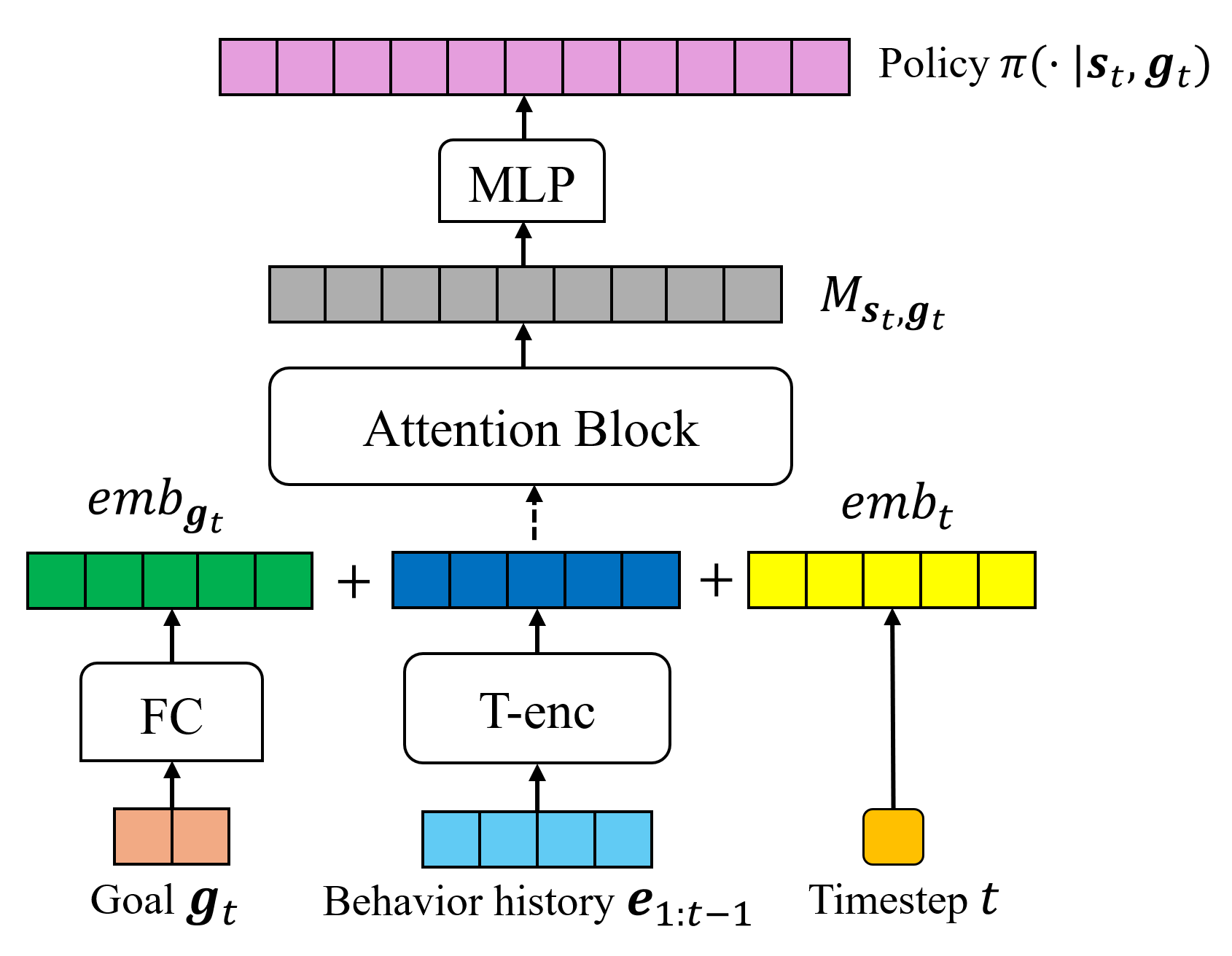}
\end{minipage}
\caption{Model structure of MOGCSL.}
\label{fig:structure}
\end{figure}

\input{sections/tables_4}

\subsection{Denoising Capability Experiments}\label{sec:denoise_exp}

To illustrate the denoising capability of MOGCSL,  we consider the same set-up introduced in Section \ref{sec:theor}, including definitions and notation of state, action, and reward.
As described before, we assume that the noisy portion of the training data originates from users who are presented with a list of items that are not suitable for them, rendering their choices for these recommendations not meaningful. Conversely, data samples with higher goals are generally less noisy, meaning $\varepsilon(a)$ is closer to $a$. We also record a long-term and multidimensional goal (e.g., the cumulative reward) $\bg=(g_1,...,g_n)$ at each interaction (known only at the end of the session, but recorded retroactively). 
% We will only consider the goals of actual recommender system, and not of the ideal recommender system, and so there is no need to introduce both a $\bg$ and a $\varepsilon(\bg)$. 
Thus our training data is a sample from a distribution $\mathcal{D}$ of tuples $(\bs,\bg, \varepsilon(a))$, where the state of the user is represented by a vector $\bs$.

As far as we know, in current real-world recommendation datasets, there are \textbf{no instance-level labels that distinguish random or weakly preference-driven interactions from genuinely preference-reflective ones}. For this reason, we use controlled noise injection and synthetic experiments to explicitly study this mechanism.

We first evaluate robustness on the real-world dataset, RetailRocket, by explicitly injecting noise into the training data. Specifically, we first randomly sample 20\% of the original training sequences from those with the lowest 20\% average goal (i.e., cumulative reward) computed over the click and purchase objectives. For each selected sequence, we replace its final ground-truth item with a uniformly random item from the catalog. This procedure introduces mislabeled interactions that are associated with low long-term goals, which fail to reflect users’ true preferences and correspond to our noise definition.

We then train the baseline methods, Share-FAMO, MMoE-FAMO, RMTL, PMORS, MOPRL, and our proposed MOGCSL, on this corrupted training set, while keeping evaluation unchanged on the original test set.

Table~\ref{tab:main_5} reports performance under noisy training as well as the relative performance drop compared to training on the original data. Across metrics, MOGCSL exhibits the smallest degradation, indicating substantially higher robustness to low-goal interaction noise. In contrast, the baselines suffer larger performance drops, suggesting greater sensitivity to noisy user behaviors that are weakly aligned with long-term 
objectives.

Besides the noise injection experiments, we further conduct experiments to study MOGCSL's denoising capability in a more fine-grained way.
As far as we know, there is no real-world recommendation dataset that contains the labels distinguishing between random and genuinely preference-driven user interactions.
As a result, we generate a dataset to study the effect as follows:
%For this purpose, there is no need for the data to be sequential. Consider the following data-generating process, from which we will sample $5000$ datapoints:
(1): The states and the goals are sampled from two independent multivariate normal distributions with dimension of $50$ and $5$ respectively (mean is 0 and variance is 2 in each dimension).
% (2): The goals are sampled from a multivariate normal distribution of dimension $5$ in a manner that is independent of the distribution of states.
(2) The ground-truth action $a$ is entirely determined by $\bs$, whose ID is set to the number of coordinates of $\bs$ that are greater than 0.
(3) Define $\varepsilon(a)$ as: $\varepsilon(a)=a$ if $g_i>-1$ for all $i$; otherwise $\varepsilon(a)$ is uniformly random.

Formally, $\varepsilon(a)$ is determined by $\bg$ and $\bs$ as follows:
\begin{equation}\small
    \varepsilon(a) = ( \prod \limits_{i=0}^{n-1} \Indi (g_i>-1)) a + (1-\prod \limits_{i=0}^{n-1} \Indi (g_i>-1)) randint[0, N-1],
    \label{eq:epsilon_a}
\end{equation}
where $a=\sum_{j=0}^{N-1} \Indi (s_j>0)$, $N$ is the dimension of states and $n$ is the dimension of goals.

%where $\bg=(g_1,...,g_n)$, and  where $\varepsilon(a)$ is $a$ if $g_i>\tilde g_i(\bs)$ for all $i$ given a fixed judge function $\tilde g_i(\bs)$; and otherwise $\varepsilon(a)$ is uniformly random. In other words, if all the goals are high then  the data is not noisy, but if one of the goals is low then the data is nonsensical.

%HERE WE SAY THAT OUR ALGORITHM IS ABLE TO COMPLETELY IGNORE THE RANDOMNESS IN THE DATA, BUT THAT NEXT ACTION PREDICTION DOES NOT. IDEALLY WE CAN SAY THIS WITH A THEOREM! REQUIRES SOME THOUGHT BUT WOULD BE WORTH WHILE.
%ALSO HERE WE WANT TO COMPARE IT WITH THE CASE WHERE WE ONLY CARE ABOUT ONE GOAL; WE'D WANT TO SAY THAT STILL YOU'D BE TAKING IN NONSENSICAL DATA BECAUSE OF THE OTHER GOALS. WE WANT TO DO THIS RIGOROUSLY. ALSO EXPLAIN THE LIMITATIONS, THAT'S VERY IMPORTANT: WHEN IS IT BETTER TO DO MULTIDIMENIONSIONAL GOALS VERSUS WHEN THOSE ASSUMPTIONS ARE BROKEN.-

\begin{table}[]
\centering
\caption{{Comparison of XGBoost with different inputs. The mean and standard deviation over 5 seeds are reported.}}
% \vspace{-5px}
\label{tab:analy}

\begin{tabular}{p{2.5cm}p{2.5cm}<{\centering}p{2.5cm}<{\centering}}
\hline
&Accuracy  & M-Logloss  \\\hline
XGBoost-s& \small{0.0576}\scriptsize{$\pm$0.0038} & \small{3.7595}\scriptsize{$\pm$0.0393} \\
XGBoost-ug& \small{0.0598}\scriptsize{$\pm$0.0009} & \small{3.7262}\scriptsize{$\pm$0.0281} \\
\hline
XGBoost-mg& \textbf{\small0.0634\scriptsize{$\pm$0.0027}} & \textbf{\small3.6603\scriptsize{$\pm$0.0097}} \\
\hline
\end{tabular}
\vspace{-3mm}
\end{table}

Since MOGCSL is applicable to any supervised model by integrating goals as additional input features, we choose a simple XGBoost classifier \cite{chen2016xgboost} here for the sake of clarity. Specifically, we train three variants of XGBoost classifier on this data to predict the action given a state: (1) XGBoost-s: this variant only takes $\bs$ as input and ignores $\bg$. It cannot detect noisy instances in $\mathcal{D}$ because it lacks access to $\bg$. 
(2) XGBoost-ug: this variant is taken as a single-objective GCSL model, which takes $\bs$ and only the first coordinate $g_1$ of $\bg$ as input. Clearly, it's also unable to precisely distinguish noisy data since $\varepsilon(a)$ is determined by all dimensions of $\bg$ (as shown in Eq. (\ref{eq:epsilon_a})). (3) XGBoost-mg: this variant is based on our multi-objective GCSL, which takes both $\bs$ and $\bg$ as input. It is the only one capable of distinguishing all the noisy data by learning the determination pattern from $\bg$ and $\bs$ to $\varepsilon(a)$.

During inference stage, for XGBoost-ug and XGBoost-mg, we adopt a simple strategy to determine the goals: directly setting each dimension of $\bg$ to 1, which serves as a high value to satisfy the condition $g_i>-1$ for the noiseless samples where $\varepsilon(a)=a$.

The results are presented in Table \ref{tab:analy}. It is evident that XGBoost-mg achieves the best performance. By incorporating multi-dimensional goals as input, XGBoost-mg can effectively differentiate between noisy and noiseless samples in the training data based on MOGCSL. During inference, when a high goal is specified as input, the model can make predictions based solely on the patterns and knowledge learned from the noiseless data.

In summary, MOGCSL demonstrates a validated inherent denoising capability, as evidenced by its superior empirical performance on both synthetic and real-world industrial datasets.

Furthermore, we'd like to highlight that the injected-noise and synthetic experiments are \textbf{controlled validations} of the specific denoising mechanism studied here. By contrast, the main experiments discussed in Section \ref{sec:performance} are \textbf{conducted on real-world datasets without manually injected noise}. In those experiments, MOGCSL consistently outperforms strong MTL baselines, including recommendation-specific ones like RMTL and PMORS, which don't explicitly model this type of noise. This is the evidence that the proposed mechanism is useful in realistic recommendation settings.

Our denoising claim is scoped to noise associated with low-goal trajectories, rather than arbitrary goal-independent noise. As discussed in Section \ref{sec:theor}, some short-term interactions in recommendation systems are weak indicators of true user preference when they are not followed by meaningful long-term engagement. Prior work~\cite{zhang2025robust, chen2021bias} similarly notes that such noisy short-term behaviors are common in real logged recommendation data. Our controlled experiments were therefore deliberately designed so that lower-goal trajectories contain more noise.

We therefore emphasize that the appendix experiments are not intended to establish robustness to arbitrary goal-independent noise. That broader setting is different from the denoising scenario we address here. Our claim is not that MOGCSL is a universal denoising method, but that it is particularly effective when noisier supervision is correlated with lower long-term outcomes.

\subsubsection{Datasets} 
We conduct experiments on two publicly available datasets: Challenge15 \footnote{https://recsys.acm.org/recsys15/challenge} and RetailRocket \footnote{https://www.kaggle.com/retailrocket/ecommerce-dataset}. 
%\footnote{https://recsys.acm.org/recsys15/challenge}  \footnote{https://www.kaggle.com/retailrocket/ecommerce-dataset}. 
They are both collected from online e-business platforms by recording users' sequential behaviours in recommendation sessions. Specifically, both datasets include binary labels indicating whether a user clicked or purchased the currently recommended item. Following previous research \cite{xin2022rethinking, xin2020self}, we filter out sessions with lengths shorter than 3 or longer than 50 to ensure data quality.

After preprocessing, the Challenge15 dataset comprises 200,000 sessions, encompassing 26,702 unique items, 1,110,965 clicks and 43,946 purchases. Similarly, the processed RetailRocket dataset consists of 195,523 sessions, involving 70,852 distinct items. It documents 1,176,680 clicks and 57,269 purchases. We partition them into training, validation, and test sets, maintaining an 8:1:1 ratio.

\subsubsection{Baseline Details} 
In the experiments, we compare two representative model architectures for multi-objective learning: 
\begin{itemize}[]
    \item \textbf{Shared-Bottom} \cite{ma2018modeling}: A classic model structure for multi-objective learning. The bottom of the model is a neural network shared across all objectives. On top of this shared base, separate towers are added for each objective, producing predictions specific to that objective.
    \item \textbf{MMOE} \cite{ma2018modeling}: A widely used multi-objective model architecture. It first maps inputs to multiple expert modules shared by all objectives. These experts contribute to each objective through designed gates. The final input for each tower is a weighted summation of the experts’ outputs.
\end{itemize}

Beyond architectural adaptations, other works focus on studying optimization constraints, mainly through adjusting weights of losses for different objectives. We compare the following methods:

\begin{itemize}[]
    \item \textbf{Fixed-Weights} \cite{wang2016multi}: A straightforward strategy that assigns fixed weights based on grid search results from the validation set. These weights remain constant throughout the whole training stage.
    \item \textbf{DWA} \cite{liu2019end}: It's designed to dynamically assign weights by considering the rate of loss change for each objective during recent training epochs. It tends to assign larger weights to objectives with slower loss changes.
    \item \textbf{PE} \cite{lin2019pareto}: It's designed for generating Pareto-efficient recommendations across multiple objectives. It's optimized for Pareto efficiency, ensuring no further improvement in one objective comes at the expense of any others.
    \item \textbf{Nash-MTL} \cite{navon2022multi}: This method frames the gradient combination step in multi-objective learning as a bargaining game and use the Nash bargaining solution to find the optimal update direction.
    \item \textbf{FAMO} \cite{liu2024famo}: This recent method aims to dynamically adjust the weights for multi-objective learning, achieving balanced task loss reduction while maintaining relatively low space and time complexity.
\end{itemize}

Note that to ensure a fair comparison, we employ the \emph{T-enc} and self-attention block introduced in Section \ref{sec:training} as the base module to encode sequential data for all compared baselines. This ensures that user sequential behavior information is captured consistently in a concise and computationally efficient manner, aligning with our core motivation of designing a lightweight, scalable, and efficient framework for multi-objective recommendation.

\subsection{Evaluation Metrics}\label{sec:metrdetail}

We employ two widely recognized information retrieval metrics to evaluate model performance in top-$k$ recommendation. Hit Ratio (HR@$k$) is to quantify the proportion of recommendations where the ground-truth item appears in the top-$k$ positions of the recommendation list \cite{hidasi2015session}. Normalized discounted cumulative gain (NDCG@$k$) further considers the positional relevance of ranked items, assigning greater importance to top positions during calculation \cite{kang2018self}. Given our dual objectives in experiments, we evaluate performance using HR@$k$ and NDCG@$k$ based on corresponding labels for click and purchase events (i.e., whether an item was clicked or purchased by the user).

\subsection{Implementation Details}
First, to ensure a fair comparison, we employ the transformer encoder and self-attention block introduced in Section \ref{sec:training} as the base module to encode the input features for all compared baselines. We preserve the 10 most recent historical interaction records to construct the state representation. For sequences shorter than 10 interactions, we pad them with a padding token. The embedding dimensions for both state and goal are set to 64, and the batch size is fixed at 256. We utilize the Adam optimizer for all models, tuning the learning rate within the range of [0.0001, 0.0005, 0.001, 0.005]. Additionally, for methods that necessitate manual assignment of weights, we fine-tune these weights in the range of [0.1, 0.2, …, 0.9] based on their performance on the validation set. The sample size $K$ in Algorithm \ref{alg:inference} is set to 20 in the experiments. 
For all baseline methods, we adhere to the implementation details, optimization configurations, and hyperparameter tuning strategies reported in their original papers.
All experiments are conducted five times, each with different random seeds, and we report the mean and standard deviation of the results. When comparing the time complexity, each experiment was conducted with a separate NVIDIA RTX 3090 and AMD 3960X.

\subsection{Comparison between MOGCSL and MOPRLs}\label{sec:moprls}

Figure \ref{fig:MOPRL} shows the performance comparison of MOGCSL to MOPRLs with different weight combinations.
% \vspace{-2mm}
\begin{figure}[H]
\centering
\begin{minipage}{0.46\linewidth}
    \centering
    \includegraphics[width=0.99\linewidth]{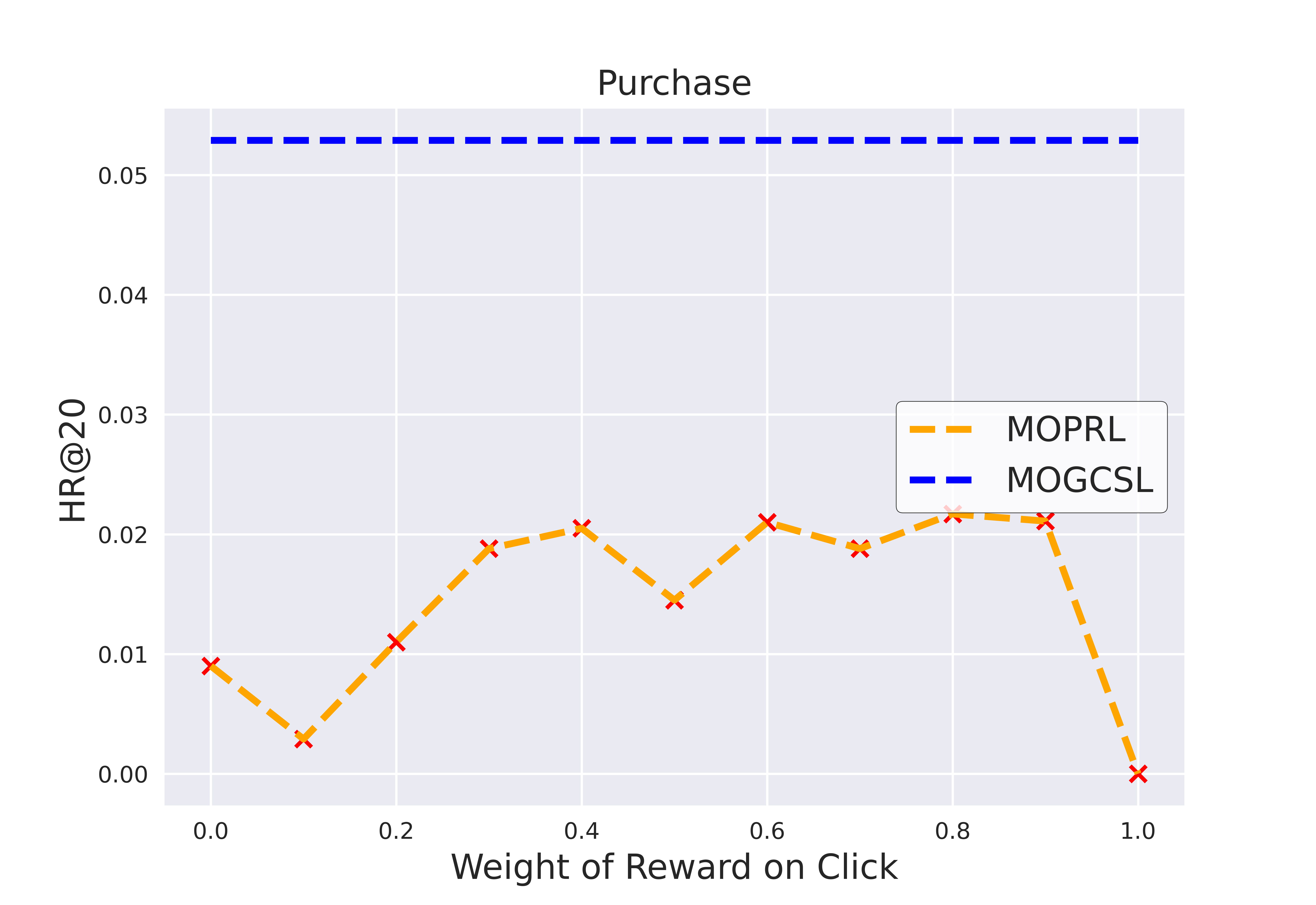}
\end{minipage}
\begin{minipage}{0.46\linewidth}
    \centering
    \includegraphics[width=0.99\linewidth]{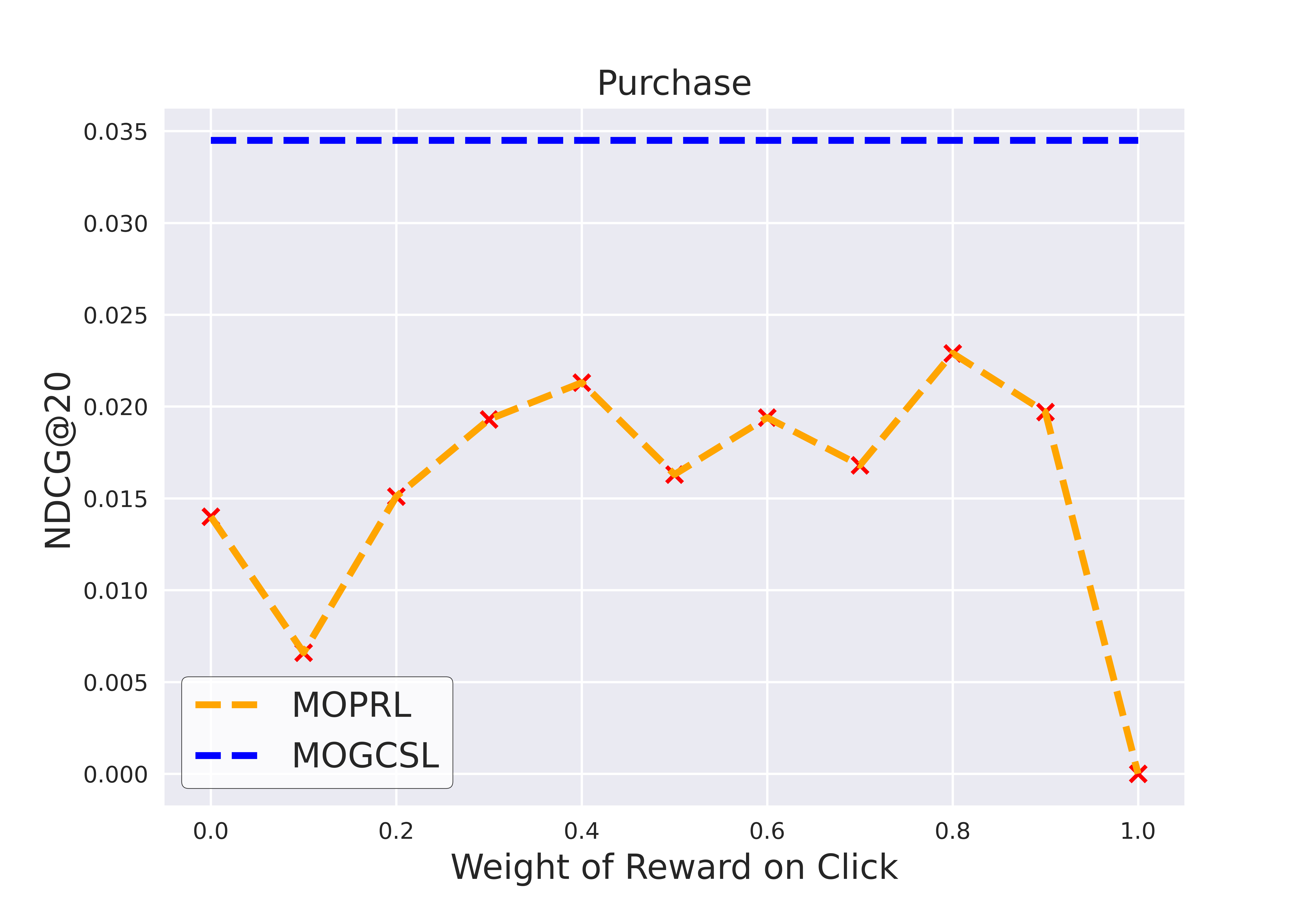}
\end{minipage}
%\qquad
%让图片换行，

\begin{minipage}{0.46\linewidth}
    \centering
    \includegraphics[width=0.99\linewidth]{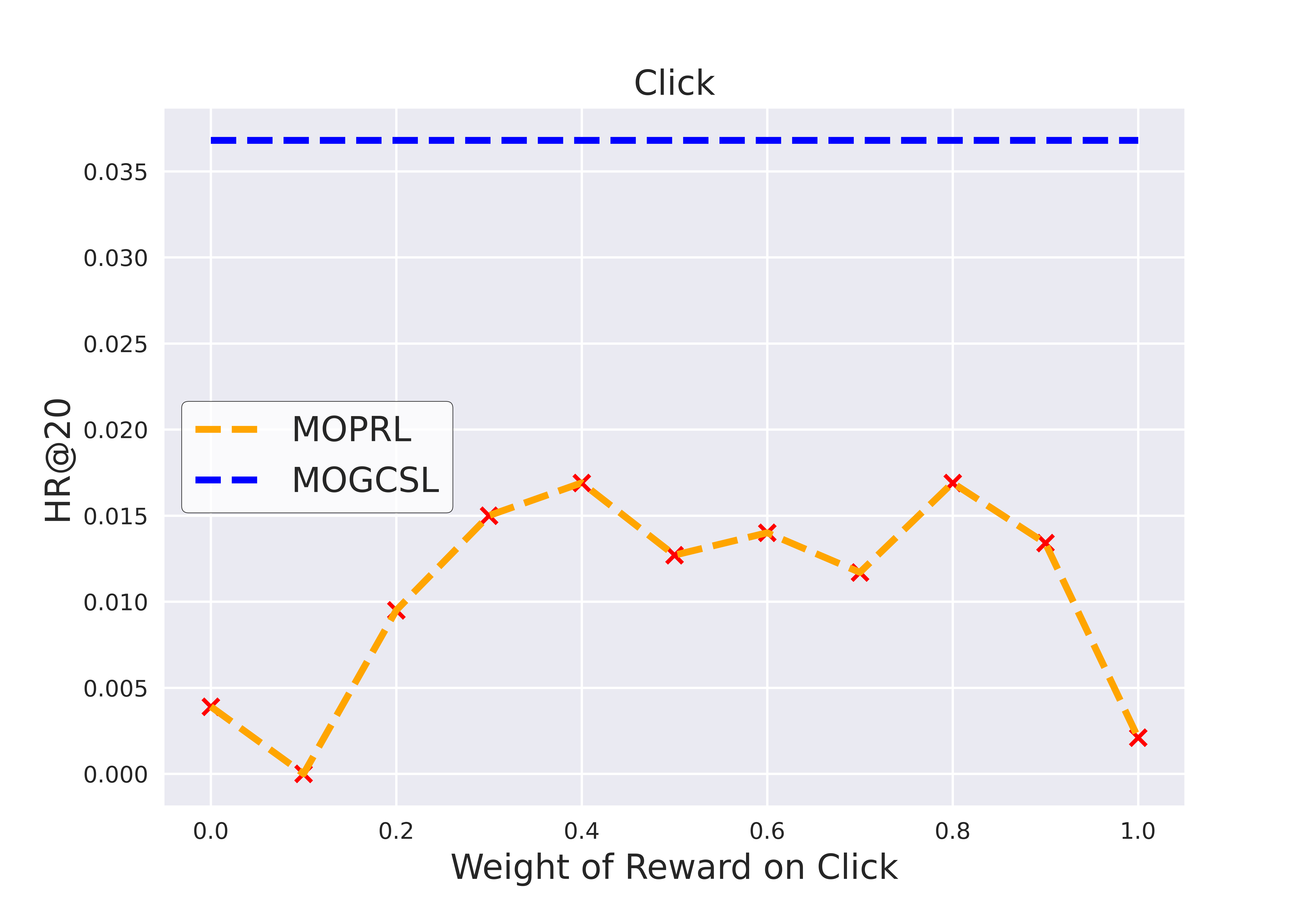}
\end{minipage}
\begin{minipage}{0.46\linewidth}
    \centering
    \includegraphics[width=0.99\linewidth]{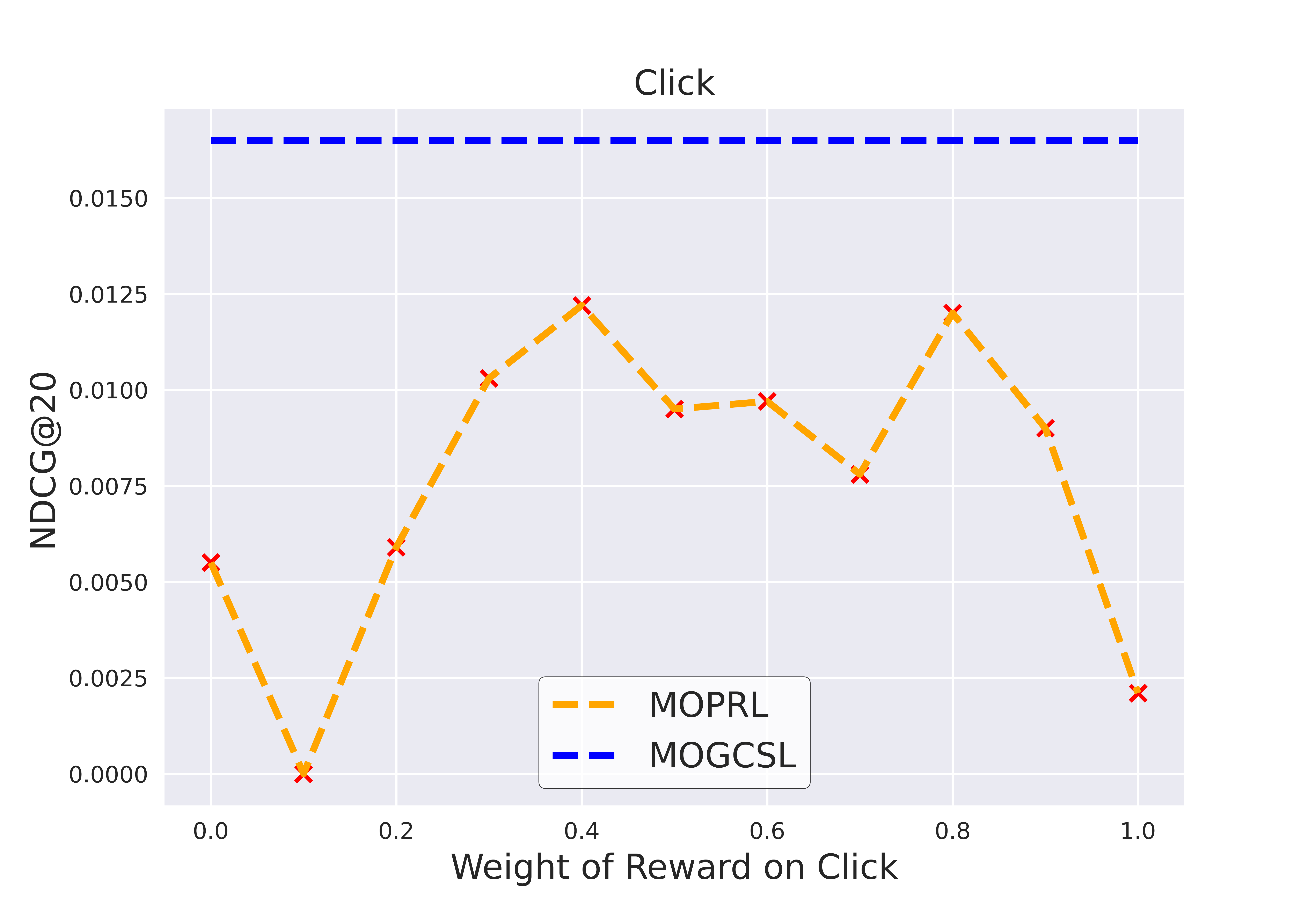}
\end{minipage}
\caption{Comparison between MOGCSL and MOPRLs with different weight combinations on RetailRocket. Performance of MOGCSL is not dependent on the weights.}
\label{fig:MOPRL}
\vspace{-3mm}
\end{figure}

% \newpage

\subsection{Comparison of MOGCSL-S w.r.t Factors}\label{sec:scalar_comp}

The detailed results of the performance comparison of MOGCSL-S with different factors for inference goal is shown in Figure \ref{fig:scalar}.

\begin{figure}[ht]
%\vspace{-5mm}
\centering
\begin{minipage}{0.48\linewidth}
    \centering
    \includegraphics[width=0.9\linewidth]{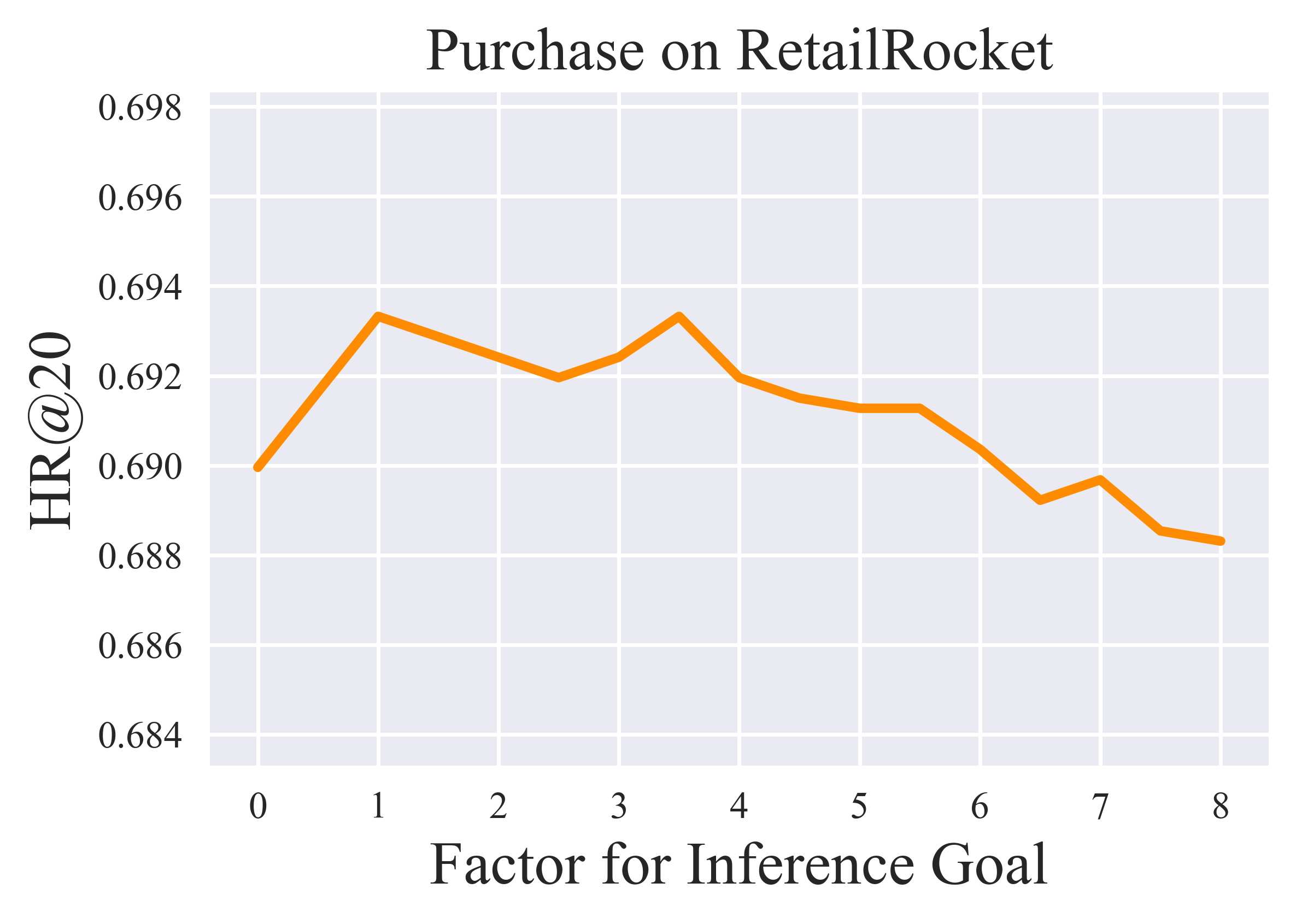}
\end{minipage}
\begin{minipage}{0.48\linewidth}
    \centering
    \includegraphics[width=0.9\linewidth]{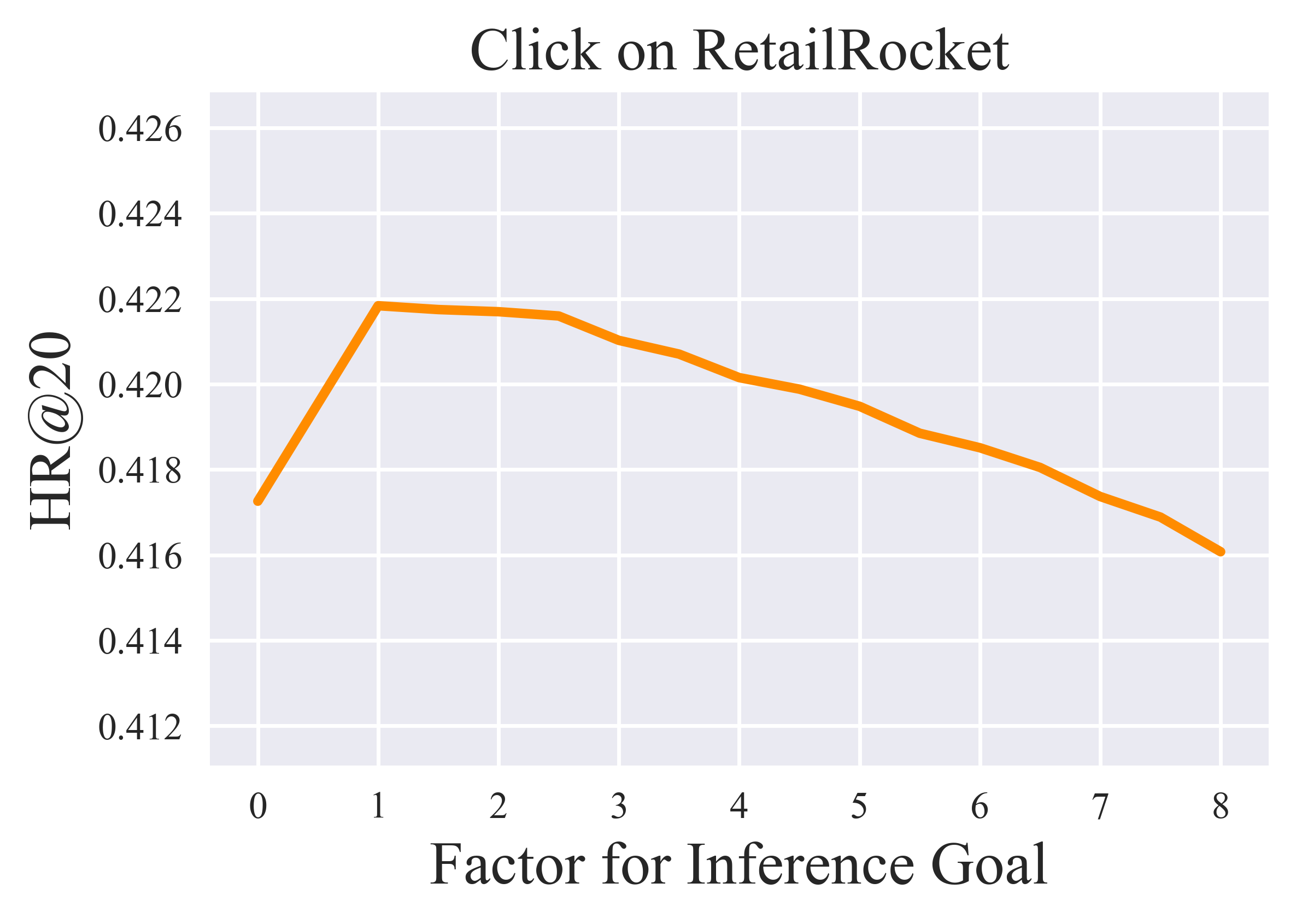}
\end{minipage}
%\qquad
%让图片换行，

\begin{minipage}{0.48\linewidth}
    \centering
    \includegraphics[width=0.9\linewidth]{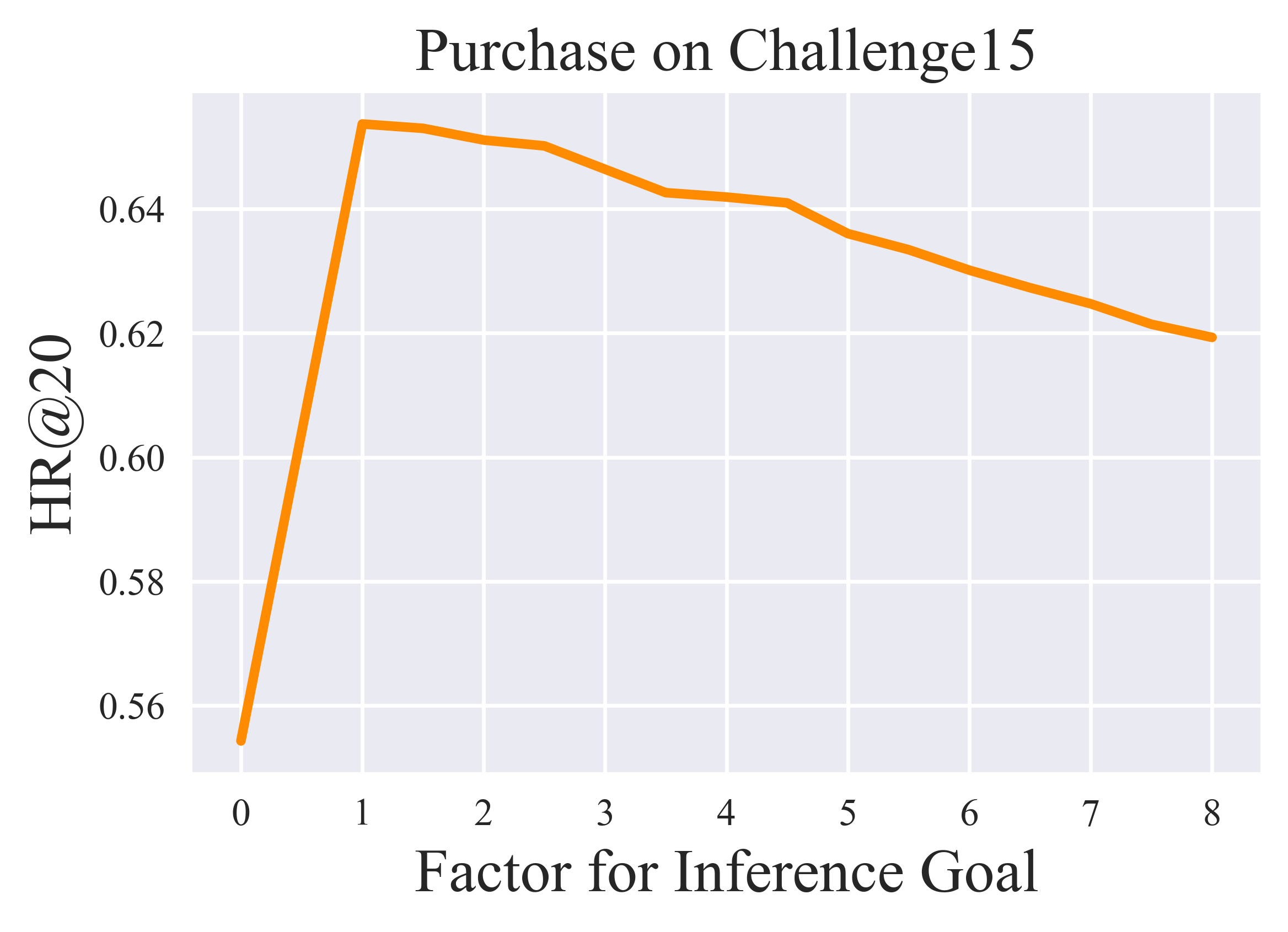}
\end{minipage}
\begin{minipage}{0.48\linewidth}
    \centering
    \includegraphics[width=0.9\linewidth]{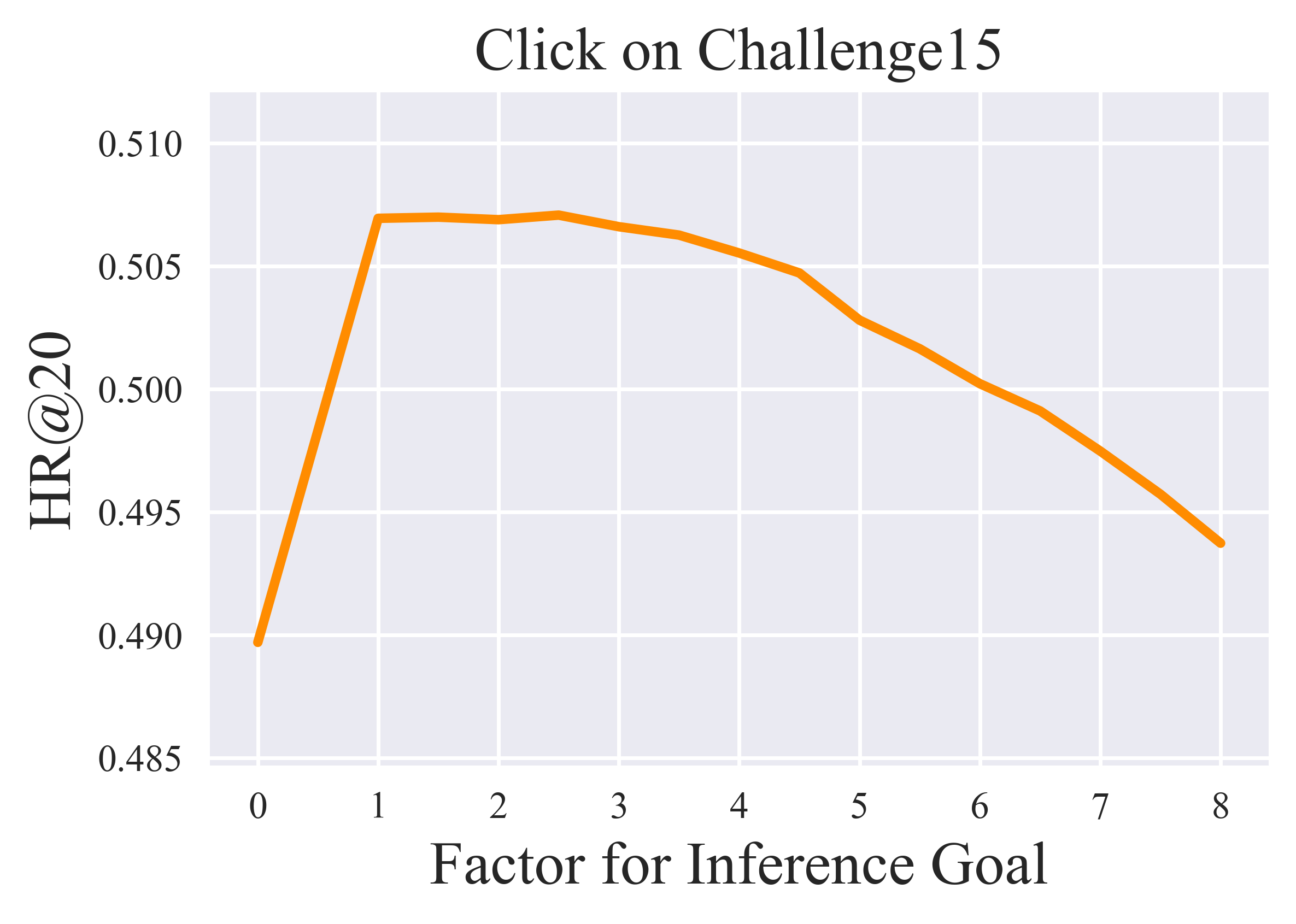}
\end{minipage}
\caption{Performance of MOGCSL-S with different factors for inference goal.}
\label{fig:scalar}
% \vspace{-2mm}
\end{figure}
% , height=2.5cm

%% file: sections/tables_4.tex
\begin{table*}[b]\small
    \centering
    \begin{threeparttable}
    \captionsetup{width=0.9\textwidth}
    \caption{Performance and relative performance drop by training on corrupted, noisy RetailRocket dataset.}
    \label{tab:main_5}

    \setlength{\tabcolsep}{2.5pt}  % default is 6pt; reduce inter-column spacing

    \begin{tabular}{@{}c c c c c c c c c@{}}
    \toprule
    \multirow{2}{*}{}
    & \multicolumn{4}{c}{Purchase (\%)}
    & \multicolumn{4}{c}{Click (\%)} \\
    \cmidrule(lr){2-5} \cmidrule(lr){6-9}
    & HR@20 & Drop Ratio & NG@20 & Drop Ratio
    & HR@20 & Drop Ratio & NG@20 & Drop Ratio \\
    \midrule

    Share-FAMO & 46.32\tiny{$\pm$0.11} & 7.71$\downarrow$ & 42.33\tiny{$\pm$0.27} & 9.51$\downarrow$ & 36.97\tiny{$\pm$0.24} & 9.19$\downarrow$ & 25.14\tiny{$\pm$0.17} & 8.08$\downarrow$ \\
    MMOE-FAMO  & 47.01\tiny{$\pm$0.19} & 8.26$\downarrow$ & 41.92\tiny{$\pm$0.18} & 11.09$\downarrow$ & 37.41\tiny{$\pm$0.17} & 9.09$\downarrow$ & 24.83\tiny{$\pm$0.23} & 6.72$\downarrow$ \\
    RMTL       & 63.21\tiny{$\pm$0.09} & 5.71$\downarrow$ & 47.60\tiny{$\pm$0.27} & 6.46$\downarrow$ & 36.92\tiny{$\pm$0.17} & 7.40$\downarrow$ & 23.82\tiny{$\pm$0.15} & 5.74$\downarrow$ \\
    PMORS      & 64.17\tiny{$\pm$0.18} & 4.86$\downarrow$ & 49.01\tiny{$\pm$0.16} & 5.88$\downarrow$ & 36.05\tiny{$\pm$0.30} & 9.56$\downarrow$ & 23.89\tiny{$\pm$0.20} & 5.61$\downarrow$ \\
    MOPRL      & 62.37\tiny{$\pm$0.11} & 3.69$\downarrow$ & 48.98\tiny{$\pm$0.14} & 5.17$\downarrow$ & 36.61\tiny{$\pm$0.21} & 6.08$\downarrow$ & 24.30\tiny{$\pm$0.17} & 4.97$\downarrow$ \\
    \midrule
    MOGCSL     & \textbf{67.92\tiny{$\pm$0.17}} & \textbf{1.96$\downarrow$} & \textbf{52.13\tiny{$\pm$0.21}} & \textbf{3.28$\downarrow$} & \textbf{39.91\tiny{$\pm$0.16}} & \textbf{4.79$\downarrow$} & \textbf{25.85\tiny{$\pm$0.20}} & \textbf{3.07$\downarrow$} \\
    \bottomrule
    \end{tabular}

    \end{threeparttable}
\end{table*}

% RMTL, PMORS, MOPRL,

%% file: sections/appen_discussion.tex
\section{Additional Discussions}

\subsection{Discussion on Complexity}\label{sec:comput}

We have provided a detailed comparison of the space
complexity and training cost of MOGCSL relative to other 
baselines, demonstrating its advantages
in terms of smaller model
size and faster training speed. Here, we give a further examination of the computational complexity of MOGCSL at inference time, as well as the effect of the CVAE module.

As outlined in Section~\ref{sec:training}, MOGCSL employs a simple policy network architecture built from basic 
attention modules combined with simple fully connected layers. Consequently, inference through the policy network incurs minimal computational cost, as further evidenced by its smaller model size compared to architectures such as MMOE or reinforcement learning frameworks (see Table~\ref{tab:complexity}).
For the CVAE-based goal-selection module, we also adopt a lightweight MLP architecture with approximately 1.6M parameters, which converges rapidly during training (within 6 epochs). We've compared the training time of MOGCSL with the CVAE module (MOGCSL-C) with other baselines. As demonstrated in Table~\ref{tab:complexity_2}, we can see that even with the additional computation introduced by the CVAE, MOGCSL-C remains among the most efficient multi-objective learning methods.

At inference, forward propagation through this CVAE introduces only negligible overhead due to the simplicity of the MLP structure. 
Although goal sampling introduces additional computation, it is required only once at the beginning of each inference sequence  in 
\begin{wraptable}{rh}{0.5\textwidth}
\renewcommand\arraystretch{1.0}
\footnotesize
\centering
\caption{Comparison of time and space complexity of MOGCSL-C on RetailRocket. }
% \vspace{-5px}
\label{tab:complexity_2}
\begin{tabular}{ccc}
\hline
&Model Size&Training time \\\hline
Share-Fix&14.0M&9.6Ks\\
Share-DWA&14.0M&5.3Ks\\
Share-PE&14.0M&5.6Ks\\
Share-Nash&14.0M&8.6Ks\\
Share-FAMO&14.0M&5.1Ks\\
MMOE-Fix&14.1M&10.2Ks\\
MMOE-DWA&14.1M&9.5Ks\\
MMOE-PE&14.1M&60.5Ks\\
MMOE-Nash&14.1M&12.4Ks\\
MMOE-FAMO&14.1M&8.8Ks\\
RMTL&17.5M&100.2Ks\\
PMORS&14.2M&10.4Ks\\
MOPRL&9.1M&3.2Ks\\
\hline
MOGCSL-C&10.7M&4.8Ks\\
\hline
\end{tabular}
\vspace{-3mm}
\end{wraptable}
practical applications. Subsequent goals are computed efficiently through iterative subtraction of current rewards, as established in our theoretical analysis. 

Another important advantage of MOGCSL lies in its scalability to higher-dimensional objectives. By representing multiple objectives as a single vector input, MOGCSL naturally integrates additional goals without increasing architectural complexity. This vectorized formulation enhances scalability relative to conventional multi-objective approaches, which typically experience rapidly growing complexity with dimensionality. Importantly, this design eliminates the need for architectural modifications or substantial computational overhead when scaling to higher dimensions, as demonstrated in Table~\ref{tab:complexity}.

\subsection{Discussion on Learning Paradigm}\label{app:learning}

MOGCSL is designed for automated and effective learning in multi-objective settings by leveraging offline trajectories conditioned on successfully achieved multi-dimensional goals. The key intuition is to reframe the challenge of balancing potentially conflicting objectives into the more tractable task of learning to achieve realistic goals directly from observed trajectories. During training, MOGCSL directly learns to reach certain goals across multiple objectives by utilizing each trajectory in the training data as a demonstration of successful achievement of the goal that it actually achieves.

At inference, MOGCSL models the distribution of achievable goals across multiple objectives conditioned on the initial state. From this distribution, it samples goals that are both realistic and aligned with practical business requirements, as specified through a customizable utility function. By always grounding goal selection in attainable outcomes, MOGCSL naturally respects inter-objective trade-offs and avoids the need for manually imposed constraints or heuristic weight tuning. This design yields a more robust and flexible framework for multi-objective optimization, readily adaptable to diverse application scenarios.

Specifically, the utility principle used in our experiments is defined in Eq. (\ref{eq:U_g}), which selects goals in a moderate manner such that no other achievable goal dominates it across all objectives. This design aligns with our aim of producing balanced and good performance across multiple objectives in the experiments.
However, this utility principle can also be flexibly adapted to prioritize different objectives based on application needs in MOGCSL.
For instance, if one objective holds greater business importance, practitioners can modify the partial ordering relationship accordingly. One of the simplest and most direct ways could be assigning higher weights to emphasize it. However, we note that this weighting mechanism differs fundamentally from loss weighting in training-based methods (e.g., MOPRL and DWA in our baselines). In MOGCSL, training is not guided by predefined objective weights; instead, the model learns policies to achieve diverse goals across all objectives demonstrated by offline trajectories. The relative importance of objectives is only applied during inference for goal selection among achievable goals, allowing the policy to flexibly align with empirical requirements. This inherent flexibility enables practitioners to tailor MOGCSL to diverse real-world scenarios and evolving objective priorities with minimal retraining effort.

\subsection{Discussion on Contribution and Novelty}

To further clarify the contribution and impact of our research, we summarize several important innovations that distinctly address critical challenges in general goal-conditional supervised learning and multi-objective recommendation scenarios:

\begin{itemize}[leftmargin=*]
    \item We are the first to integrate an MOMDP approach into GCSL's framework, extending the goal from a scalar to a multidimensional vector. This extension is non-trivial since defining and selecting achievable "high" multi-dimensional goals is conceptually complex and previously unexplored. Our paper proposes and evaluates both principled and heuristic goal-selection approaches, demonstrating their effectiveness across different datasets.

    \item We provide the first theoretical characterization of achievable goal distributions within the multi-objective setting of GCSL, formalized through Theorem 1 with complete proof (Section~\ref{sec:inference} and Appendix~\ref{sec:proof}).

    \item Building upon this theoretical foundation, we developed a novel goal-choosing algorithm that models the distribution of achievable goals over interaction sequences to select desirable high goals for inference. This algorithm is novel even for the single-objective case and addresses a fundamental challenge in GCSL applications.

    \item We were the first to demonstrate GCSL's capability to mitigate the harmful effects of noisy instances, which are common in real-world recommendation data. We showed that MOGCSL can leverage the future impact of current actions across multiple objectives to achieve this ability (Section~\ref{sec:theor} and Appendix~\ref{sec:denoise_exp}).

\end{itemize}

The novelty of this work is not simply replacing a scalar goal with a vector. Rather, our contribution is to \textbf{extend GCSL to the multi-objective setting and address the inference-time challenge that then arises}: how to define and select achievable high goals across multiple objectives.

This issue is largely absent in scalar-goal GCSL. In the scalar setting, inference-time goal selection can often be handled with simple heuristics or validation-based tuning, since only a single target value needs to be searched; this is exactly what standard GCSL approaches do, as discussed in Section \ref{sec:inference}. In the multi-objective setting, however, the goal becomes a vector of coupled objectives, and naive tuning can become impractical as dimensionality grows. Moreover, collapsing multiple objectives into a scalar requires manually chosen weights, which impose a fixed trade-off and discard the original per-objective structure.

Our use of vector-valued goals is therefore not merely representational. It preserves the objectives and makes achievability modeling and principled goal selection essential. Our framework addresses this by modeling achievable-goal distributions and proposing inference-time goal-selection strategies. To the best of our knowledge, this is \textbf{the first work to study this issue systematically in GCSL for multi-objective learning}. Extending scalar goals to vector-valued goals may appear straightforward, but modeling achievability and selecting goals jointly across multiple dimensions are far from trivial.

\input{sections/tables_5}

\subsection{Discussion on Relative Performance}\label{sec:rela}

As shown and discussed in Section \ref{sec:goal-generation}, MMOE-PE optimizes for strict Pareto efficiency by assuming that improving one objective should not degrade performance on any other objective. Under this criterion, MMOE-PE achieves slightly better click performance than MOGCSL on Challenge15, with gains of about 1-2\%. However, MOGCSL delivers much larger improvements on the more important purchase metrics, with gains of 11-20\%. This suggests that, on the Challenge15 dataset, optimizing under a strict Pareto-efficiency constraint severely limits performance on the valuable purchase objective.

We then analyze this phenomenon from the perspective of dataset characteristics. In particular, we compute the Pearson correlation between click and purchase signals on both RetailRocket and Challenge15. The correlation is 0.72 on RetailRocket, but only 0.33 on Challenge15. This indicates that click and purchase behaviors are much more aligned on RetailRocket, whereas they exhibit substantially greater conflict on Challenge15. As a result, while MMOE-PE can maintain a slight advantage on click-related metrics, its strict Pareto-efficiency requirement restricts its ability to optimize purchase performance on Challenge15. By contrast, MOGCSL does not rely on a strict Pareto-efficiency constraint in either learning or objective decision, which enables a more practically useful trade-off across objectives. In recommendation settings, a substantial improvement on high-value objectives (e.g., purchase) with only minor reductions on less critical objectives (e.g., click), is often preferable in real-world applications.

% We would like to emphasize that this is exactly the phenomenon we aim to highlight. In multi-objective recommendation, the degree of alignment or conflict between objectives is a core property of the problem. Our analysis shows a clear difference between RetailRocket and Challenge15, and the relative behavior of MOGCSL and MMOE-PE changes consistently with that difference: when click and purchase are more aligned, the methods exhibit more similar trends; when they conflict more strongly, MMOE-PE keeps a slight click advantage while MOGCSL achieves much larger gains on purchase. This pattern is consistent with the design difference between a strict Pareto-efficiency method and our utility-based goal-selection framework.

To further investigate this phenomenon, we also compare MOGCSL with MMOE-PE on Tenrec, where the Pearson correlation between like and click is 0.52, indicating an intermediate level of objective alignment between Challenge15 and RetailRocket. Table \ref{tab:main_7} shows the results. There is a consistent trend across all three datasets: when objective conflict is strongest, as in Challenge15, MOGCSL achieves the largest gain over PE on the more valuable objective (e.g., around 20\% improvement on purchase HR@10), while only slightly sacrificing click performance; when the objectives are more aligned, as in RetailRocket, the gain on purchase remains substantial (e.g., around 17\% on purchase HR@10) and the click metrics are nearly the same as PE; and on Tenrec, whose conflict level lies between the other two datasets, MOGCSL remains on par with PE on click while still improving like by about 19\% relatively.
Taken together, these results suggest that the advantage of MOGCSL becomes more pronounced on the higher-value objective as objective conflict becomes stronger, whereas under weaker conflict it can improve the valuable objective without noticeably harming the secondary one. Overall, compared with strict Pareto-efficiency-based methods such as PE, MOGCSL is better suited to practical recommendation settings where different objectives carry different values.

\subsection{Discussion on Effects of Goal Value}\label{sec:eff_goal}

Goal selection is central to MOGCSL, since inference performance directly depends on those choosing goals. Our experiments analyze this issue from several perspectives.

First, the value of the inference goal matters. As shown in Section 4.4 and Appendix B.7, sweeping the scaling factor of inference goals reveals a clear non-monotonic pattern: performance is best when the goal is set in a moderate range, while goals that are too small or too large consistently hurt performance. This is consistent with our formulation, where inference selects goals from the achievable set according to the utility principle in Eq.(8) \footnote{{Note that there may be multiple samples satisfying this condition. In such cases, we select one with highest purchase.}}
, rather than blindly pursuing arbitrarily large targets. In this sense, MOGCSL is designed to operate within the achievable region and to select reasonable trade-offs among objectives.

Second, the CVAE-based goal-selection module should not be interpreted as choosing a single global goal shared across all examples. Instead, it models a \emph{state-conditioned} goal distribution and selects personalized inference goals for each initial state. For this reason, a conventional grid search over a small set of global goals is not directly comparable to the CVAE-based mechanism. The closest statistical alternative is exactly MOGCSL-S, which sets the inference goal at each timestep to the average cumulative reward in the offline data at the same timestep, scaled by a factor $\lambda$ chosen by grid search on the validation set. Section \ref{sec:goal-generation} and Appendix \ref{sec:scalar_comp} already compare this statistical strategy with the CVAE-based method in detail. More broadly, our experiments also include comparisons with several grid-search-based baselines, such as MOPRL, Fixed-Weight, and RMTL, and MOGCSL performs favorably against these alternatives.

Third, our current experiments already evaluate the boundary behavior of target goals. The $\lambda$-sweep in Appendix \ref{sec:scalar_comp} shows that both overly ambitious and overly conservative targets degrade performance. Moreover, the CVAE-based procedure is explicitly designed to avoid arbitrary extrapolation: it samples candidate goals conditioned on 
the initial state and evaluates them through the estimated distribution of achievable outcomes before selecting one according to the utility function. Therefore, the method is not intended to chase unattainable goals, but to choose desirable goals within the achievable region. Our analysis in Section \ref{sec:goal-generation} 
and Appendix B.5 further suggests that when high-goal trajectories are sparse, more optimistic goal-generation can become less effective because the policy lacks sufficient supervision in that region.

Last, we'd like to highlight that MOGCSL is not designed to recover
the entire Pareto frontier or to optimize Pareto efficiency as its primary criterion. In recommender systems, 
different objectives often carry 
different practical values, so the 
central question is not whether a method can cover all Pareto-efficient operating points, but whether it can achieve robust and 
satisfactory performance under application-specific utility preferences. For example, in many recommendation scenarios, a substantial improvement on a higher-value objective such as purchase may be preferable even if it is accompanied by a small reduction on a lower-value objective such as click. Under such settings, evaluating the whole Pareto frontier is not always the most meaningful criterion, because some Pareto-valid regions may correspond to trade-offs that are not practically desirable. In particular, operating points with high click but low purchase may still be Pareto-efficient, yet remain unattractive in realistic recommendation environments. For this reason, our evaluation does not treat Pareto-front recovery as the main goal. Instead, MOGCSL
\begin{wrapfigure}{r}{0.5\textwidth}
\vspace{-1mm}
\centering
\footnotesize
\includegraphics[width=0.95\linewidth]{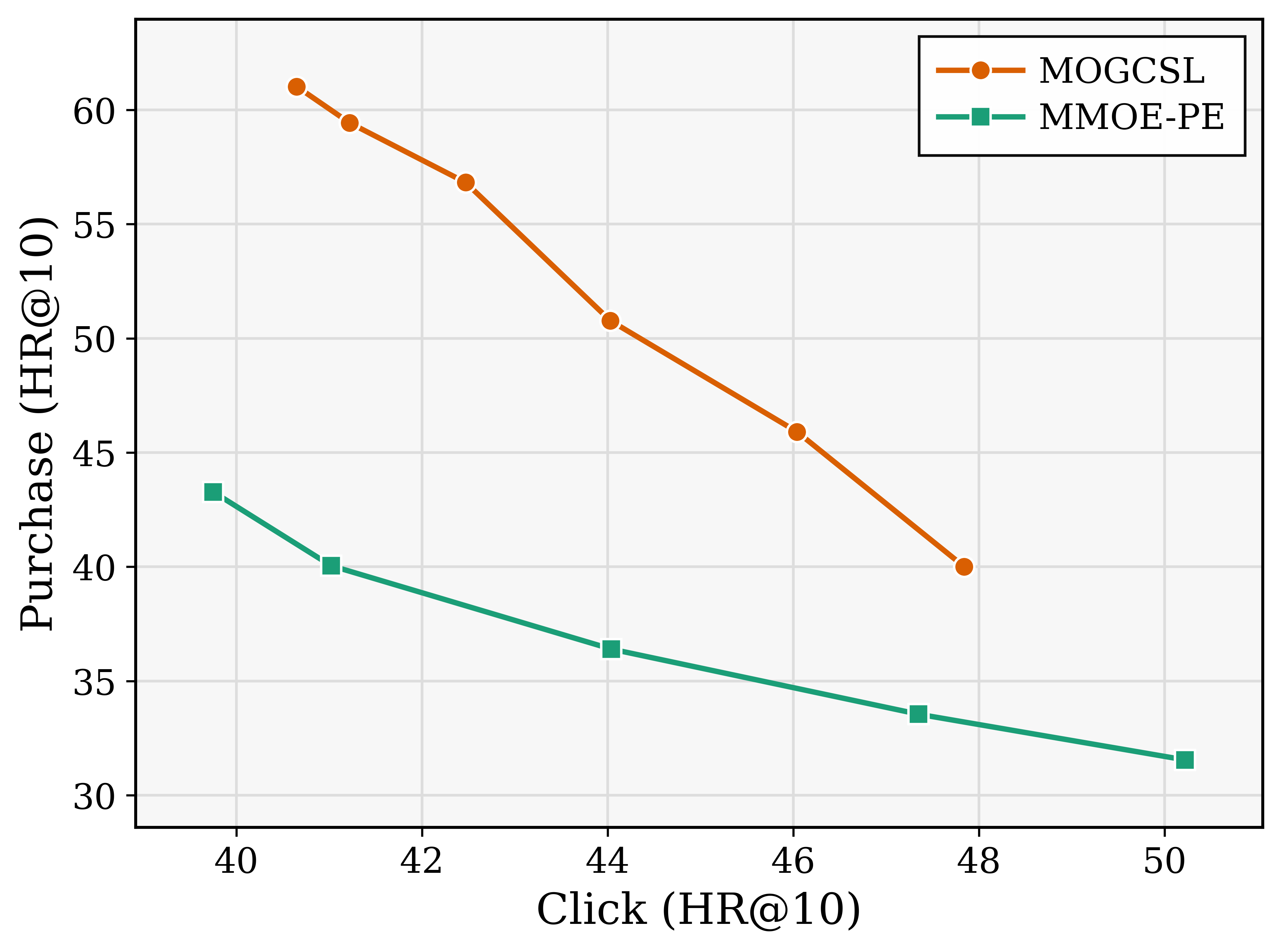}
\caption{Comparison of Pareto front between MOGCSL and MMOE-PE.}
\label{fig:pareto}
\vspace{-1mm}
\end{wrapfigure} 
learns the distribution of achievable multi-objective goals and selects operating points from the achievable set according to a customizable utility principle. The resulting analysis therefore focuses on how different achievable-goal selections affect downstream trade-offs and performance, which is more closely aligned with the practical role of goal selection in our application. This experimental design is also consistent with recent work on multi-objective recommendation, such as PMORS \cite{jin2024pareto} and RMTL \cite{liu2023multi}, which likewise do not report full Pareto-front comparisons, even though PMORS uses Pareto-based optimization techniques. Nevertheless, for completeness, we also \textbf{provide a Pareto-front comparison between MOGCSL and MMOE-PE} on click and purchase for RetailRocket. Specifically, we vary the utility used in MOGCSL goal selection so as to adjust the relative importance of click and purchase to draw this figure. The results in Figure~\ref{fig:pareto} further demonstrate MOGCSL’s advantage in covering high-value areas with strong purchase signals, whereas MMOE-PE’s purchase performance is largely constrained by its reliance on click learning.

\subsection{Discussion on Goal-Selection Strategy}\label{sec:goal_selec}

We'd like to further discuss our proposed principled goal-selection strategy in Section \ref{sec:inference}. Specifically, we use two generative models for the two distributional subproblems in Algorithm \ref{alg:inference}.

First, we use a generative model for the prior $q(g'| s')$. This is necessary because candidate-goal generation is inherently one-to-many: for a new state $s'$, there are multiple feasible trade-offs across objectives, and Algorithm 2 requires multiple candidate goals for comparison.

Second, we use a generative model for $P(g_a| s', g', \pi)$ because, under a fixed condition $c=(s', g', \pi)$, there exists only one realized sample from the condition-specific distribution $P(g|c_i)$ in the training data. Since $c$ is high-dimensional and effectively unique per instance, directly regressing to one target treats a single noisy draw as an approximation of the multi-dimensional conditional mean. In contrast, learning the conditional distribution allows $\mathbb{E}[g_a|s', g', \pi]$ to be estimated directly by sampling, which better matches the distributional nature of Algorithm \ref{alg:inference}.

That said, from a theoretical standpoint, we agree that a regression-based approach could also be applied to approximate the second component. In principle, with unlimited training data, direct regression can converge to the true expected value.
However, in Section \ref{sec:inference}, we adopt a generative formulation to remain consistent with the inference procedure. The intuition is that directly learning the conditional distribution itself and estimating its mean by sampling could be more robust with lower variance than directly regressing to the mean from only one noisy realized sample for each condition on limited offline data.

To examine this point, we conduct additional experiments comparing MOGCSL with the CVAE-based strategy against a variant that directly regresses the conditional mean of the achievable goal. Table \ref{tab:main_6} reports the results. Overall, the CVAE-based version generally performs better, although the performance gap is not large in every case. More importantly, the regression-based variant exhibits substantially higher variance across runs than MOGCSL-C. This supports our intuition that direct regression can provide a noisier estimate of the expectation, whereas the CVAE-based sampling strategy is more robust.

\input{sections/tables_6}

\subsection{Discussion on Generalizability and Scalability}\label{sec:general}

From the experiments and discussions in Section~\ref{sec:goal-generation}, we observed that different goal-choosing strategies perform differently on datasets with varying reward distributions. However, we emphasize our finding that simpler, statistics-based goal-selection strategies already outperform existing baselines across diverse scenarios, demonstrating MOGCSL's practical applicability especially where lower latency or reduced inference complexity is desired. The advanced CVAE-based algorithm provides additional performance gains specifically on datasets with higher goal values (as shown in Appendix B.8), offering practitioners the flexibility to choose strategies based on dataset characteristics and performance-complexity requirements. As a result, we highlight that this \textbf{doesn't weaken the framework's practical value}. Instead, it helps for appropriate strategy selection tailored to practical use cases. As discussed in Section \ref{sec:goal-generation}, this finding is also consistent with recent GCSL theoretical analyses \cite{bairebalancing}, which demonstrates that overly optimistic goal conditioning can be harmful when high-quality trajectories are insufficient.

From another perspective, although the formulation of MOGCSL is presented in a general manner for sequential decision-making problems, the main claim of this paper is that MOGCSL provides a general \textbf{supervised framework for multi-objective recommendation}, which is also the motivation and application setting emphasized in the abstract and introduction. We prefer to keep the scope of this work focused on recommender systems and leave broader cross-domain validation to future work. We've included a relevant discussion of further extension and application of MOGCSL into other areas in Appendix \ref{sec:limit}.

For the scalability of MOGCSL, we briefly discussed this in Appendix \ref{sec:comput}. Specifically, MOGCSL represents multiple objectives as a single goal vector, so additional objectives can be incorporated without changing the architecture. This design is naturally more scalable than many conventional MTL methods, whose complexity often grows substantially with the number of objectives. Likewise, our goal-selection mechanism is based on the achievable goal set and is naturally extensible to higher-dimensional goals. While additional objectives may require more training data to effectively adapt to higher-dimensional spaces, this challenge is intrinsic to all MTL methods and not unique to ours.

That said, our experimental design follows the standard settings used in prior multi-objective recommendation work, such as PE \cite{lin2019pareto} and PMORS \cite{jin2024pareto}, which focus on two-objective experiments. This choice ensures that comparisons remain as direct and fair as possible. In practice, these prior studies are also evaluated only under two-objective recommendation settings, which is why our empirical study adopts the same scope.

At the same time, our goal in this paper is to address the typical and practically important challenges in multi-objective recommendation, rather than to comprehensively study high-dimensional multi-objective learning. To keep this submission focused on recommendation settings, we leave a comprehensive study of higher-dimensional objective spaces to future work.

\subsection{Discussion on Generative Models}

As illustrated in Section~\ref{sec:inference}, we chose CVAE for the goal-selection algorithm primarily because of its simplicity, efficiency, and stable training dynamics, which align closely with our objective of developing scalable methods suitable for real-world recommendation scenarios. Apart from that, we've conducted some preliminary experiments with diffusion-based methods such as DDPO, and the results showed no significant performance improvement over CVAE. Additionally, CVAE demonstrated consistently stable training dynamics and faster inference, essential attributes for practical deployments.

\subsection{Discussion on Data Robustness}

As illustrated in Section~\ref{sec:training},  MOGCSL assumes the availability of reward signals within sequential data. In recommender systems, such reward signals typically derive naturally and easily from explicit user feedback, such as clicks, purchases, or engagement durations.

However, clear rewards might not be readily available in some other scenarios. For these situations, we propose the following discussions:

\begin{itemize}[leftmargin=*]
    \item Proxy rewards: Intermediate signals often effectively serve as proxies. For instance, browsing duration, add-to-cart events, or return visits in e-commerce scenarios can reliably approximate user satisfaction when explicit purchase data is sparse.

    \item Synthetic rewards: In cases like path planning or robotic control, synthetic metrics such as proximity to target states or successful task completion checkpoints can be constructed to substitute explicit rewards.

    \item Reward inference: Techniques such as inverse reinforcement learning or reward inference from demonstration could integrate seamlessly into MOGCSL, enabling application in less structured environments.

\end{itemize}

Furthermore, as discussed in Section~\ref{sec:theor}, one of the MOGCSL's distinct advantages lies in leveraging long-term goals to effectively handle potentially noisy immediate reward signals. Our experiments in Appendix~\ref{sec:denoise_exp} demonstrate that this mechanism significantly enhances robustness, enabling effective learning even when datasets contain ambiguous or imprecise signals.

\subsection{Discussion on Social Impact}

In general, we believe that MOGCSL poses limited privacy and social risk, since it can be trained using standard anonymized user interaction data and does not require collecting additional sensitive information. Moreover, by explicitly modeling multiple objectives, MOGCSL may help recommendation systems better balance different user and platform goals, such as engagement, satisfaction, diversity, and long-term utility.

Nevertheless, as with other recommendation methods, MOGCSL may inherit biases present in the training data. If historical user interactions reflect popularity bias, unfair exposure, or narrow preference patterns, the learned policy may amplify these effects. Therefore, when deploying MOGCSL in real-world systems, it is important to incorporate appropriate safeguards, such as privacy-preserving data processing, fairness-aware evaluation, and monitoring of unintended recommendation behaviors. We leave a more comprehensive study of these social impacts to future work.

%% file: sections/tables_5.tex
\begin{table*}[t]\small
    \centering
    \begin{threeparttable}
    \caption{Comparison between PE and MOGCSL on Tenrec dataset.}
    %  on RetailRocket and Challenge15. The mean and standard deviation over 5 seeds are reported. Boldface denotes the best results.}
    %: RCF outperforms all baselines on all metrics with significance level $p$-value < 0.05 (indicated by $$). The last row shows the $p$-value of comparing RCF with the best baseline on the corresponding metric (indicated by boldface).}\vspace{-5pt}
    \label{tab:main_7}

    \begin{tabular}{p{1.8cm}p{0.65cm}<{\centering}p{0.65cm}<{\centering}p{0.65cm}<{\centering}p{0.65cm}<{\centering}p{0.65cm}<{\centering}p{0.65cm}<{\centering}p{0.65cm}<{\centering}p{0.65cm}<{\centering}}
    \toprule
    \multirow{2}{*}{}&\multicolumn{4}{c}{Like (\%)}&\multicolumn{4}{c}{Click (\%)}\cr
    \cmidrule(lr){2-5} \cmidrule(lr){6-9}
    &\multicolumn{1}{c}{HR@10}&\multicolumn{1}{c}{NG@10}&\multicolumn{1}{c}{HR@20}&\multicolumn{1}{c}{NG@20}&\multicolumn{1}{c}{HR@10}&\multicolumn{1}{c}{NG@10}&\multicolumn{1}{c}{HR@20}&\multicolumn{1}{c}{NG@20}\cr
    \midrule
    MMOE-PE & {5.38\tiny{$\pm$0.22}} & {1.94\tiny{$\pm$0.19}} & {6.14\tiny{$\pm$0.35}} & {2.47\tiny{$\pm$0.28}} & {4.81\tiny{$\pm$0.18}} & \textbf{{1.59}\tiny{$\pm$0.20}} & {{5.60}\tiny{$\pm$0.19}} & \textbf{{1.98}\tiny{$\pm$0.23}}  \\
    
    MOGCSL-S & \textbf{5.96\tiny{$\pm$0.17}} & \textbf{2.15\tiny{$\pm$0.12}} & \textbf{6.93\tiny{$\pm$0.20}} & \textbf{2.70\tiny{$\pm$0.11}} & \textbf{4.87\tiny{$\pm$0.13}} & {{1.52}\tiny{$\pm$0.08}} & \textbf{{5.67}\tiny{$\pm$0.14}} & {{1.95}\tiny{$\pm$0.11}}  \\

    \bottomrule
    \end{tabular}

    \end{threeparttable}
\end{table*}

%% file: sections/tables_6.tex
\begin{table*}[t]\small
    \centering
    \begin{threeparttable}
    \caption{Comparison between regression and CVAE-based strategy for mean estimation.}
    %  on RetailRocket and Challenge15. The mean and standard deviation over 5 seeds are reported. Boldface denotes the best results.}
    %: RCF outperforms all baselines on all metrics with significance level $p$-value < 0.05 (indicated by $$). The last row shows the $p$-value of comparing RCF with the best baseline on the corresponding metric (indicated by boldface).}\vspace{-5pt}
    \label{tab:main_6}
     \begin{tabular}{p{1.8cm}p{0.65cm}<{\centering}p{0.65cm}<{\centering}p{0.65cm}<{\centering}p{0.65cm}<{\centering}p{0.65cm}<{\centering}p{0.65cm}<{\centering}p{0.65cm}<{\centering}p{0.65cm}<{\centering}}
    \toprule
    \multirow{2}{*}{[RetailRocket]}&\multicolumn{4}{c}{Purchase (\%)}&\multicolumn{4}{c}{Click (\%)}\cr
    \cmidrule(lr){2-5} \cmidrule(lr){6-9}
    &\multicolumn{1}{c}{HR@10}&\multicolumn{1}{c}{NG@10}&\multicolumn{1}{c}{HR@20}&\multicolumn{1}{c}{NG@20}&\multicolumn{1}{c}{HR@10}&\multicolumn{1}{c}{NG@10}&\multicolumn{1}{c}{HR@20}&\multicolumn{1}{c}{NG@20}\cr
    \midrule
    MOGCSL-R &  {64.88\tiny{$\pm$0.27} }&  {52.76\tiny{$\pm$0.24} }& {69.37\tiny{$\pm$0.30}} & {53.66\tiny{$\pm$0.19}} & {35.98\tiny{$\pm$0.33}} & 24.94\tiny{$\pm$0.18} & {41.07\tiny{$\pm$0.49}} & 26.21\tiny{$\pm$0.37} \\
    MOGCSL-C & {65.01\tiny{$\pm$0.07} }& {52.89\tiny{$\pm$0.04} }&  {69.34\tiny{$\pm$0.05}} &  {54.00\tiny{$\pm$0.04}} &  {36.54\tiny{$\pm$0.02}} &  {25.41\tiny{$\pm$0.04}} &  {42.20\tiny{$\pm$0.03}} &  {26.84\tiny{$\pm$0.06}} \\

    \bottomrule
    \end{tabular}

    \vspace{0.2cm}

    \begin{tabular}{p{1.8cm}p{0.65cm}<{\centering}p{0.65cm}<{\centering}p{0.65cm}<{\centering}p{0.65cm}<{\centering}p{0.65cm}<{\centering}p{0.65cm}<{\centering}p{0.65cm}<{\centering}p{0.65cm}<{\centering}}
    \toprule
    \multirow{2}{*}{[Challenge15]}&\multicolumn{4}{c}{Purchase (\%)}&\multicolumn{4}{c}{Click (\%)}\cr
    \cmidrule(lr){2-5} \cmidrule(lr){6-9}
    &\multicolumn{1}{c}{HR@10}&\multicolumn{1}{c}{NG@10}&\multicolumn{1}{c}{HR@20}&\multicolumn{1}{c}{NG@20}&\multicolumn{1}{c}{HR@10}&\multicolumn{1}{c}{NG@10}&\multicolumn{1}{c}{HR@20}&\multicolumn{1}{c}{NG@20}\cr
    \midrule
    MOGCSL-R &  {54.97\tiny{$\pm$0.37}} &  {35.07\tiny{$\pm$0.33}} &  {63.73\tiny{$\pm$0.25}} &  {37.25\tiny{$\pm$0.25}} &  {42.03\tiny{$\pm$0.19}} &  {{25.18}\tiny{$\pm$0.33}} &  {{49.96}\tiny{$\pm$0.20}} &  {{27.29}\tiny{$\pm$0.17}}  \\
    MOGCSL-C & {55.13\tiny{$\pm$0.07} }& {35.04\tiny{$\pm$0.02} }& {63.98\tiny{$\pm$0.04}} & {37.30\tiny{$\pm$0.03}} & {42.14\tiny{$\pm$0.04}} & 25.37\tiny{$\pm$0.07} & {50.12\tiny{$\pm$0.09}} & 27.53\tiny{$\pm$0.05} \\

    \bottomrule
    \end{tabular}

    \end{threeparttable}
\end{table*}

%% file: sections/appen_limit.tex
\section{Limitations and Future Works}\label{sec:limit}

Although MOGCSL demonstrates strong effectiveness and efficiency for multi-objective recommendation, we identify some potential limitations and corresponding future works. 

First, the most prominent gains from the CVAE-based goal-selection strategy are observed on datasets with relatively high average goals. Simpler statistics-based strategies remain effective when average goals are low and offer practitioners flexibility in trading off performance against computational complexity. Nonetheless, further improving the CVAE-based goal-selection mechanism is a valuable direction. Potential approaches include data augmentation techniques (e.g., re-/down-sampling, active learning) to alleviate sparsity in the high-goal region. In addition, extending MOGCSL to an online training paradigm is a promising way to further mitigate this issue if an online environment is accessible: one can pre-train on offline data, deploy the policy to collect new trajectories with higher achieved goals, and iteratively retrain on this enriched dataset.

Second, our current experiments focus on using long-term multi-objective outcomes to downweight noisy instantaneous behaviors that are weak indicators of positive preference. A natural extension is to directly model and leverage explicit negative signals. This could be realized by encoding negative feedback (e.g., quits, skips, dislikes) as additional dimensions in the goal vector, and then incorporating penalties in the utility function to prioritize minimizing those signals during inference.

Third, although our evaluation uses real-world industrial recommendation datasets, we don't conduct online A/B testing with live users due to resource limitations. Industrial deployment and online experimentation would provide additional evidence of practical impact and constitute an important direction, especially from an industry perspective.

Finally, while our study is motivated by empirical challenges in large-scale recommender systems (e.g., efficiency constraints and noisy interaction data), the MOGCSL formulation is broadly applicable to other multi-objective learning tasks. By reframing multi-objective optimization as learning to reach achievable multi-dimensional goals, the framework can potentially be applied to other domains, such as medical prediction, robotic control, and marketing optimization.

\section{Asset Licenses}\label{app:licenses}

Challenge15~\cite{ben2015recsys} is released under the Creative Commons Attribution 4.0 International License (CC BY 4.0).
The RetailRocket~\cite{roman_zykov_noskov_artem_anokhin_alexander_2022} dataset is available on Kaggle under the Creative Commons Attribution-NonCommercial-ShareAlike 4.0 International License (CC BY-NC-SA 4.0).
Tenrec~\cite{yuan2022tenrec} is released under the Creative Commons Attribution-NonCommercial 4.0 International License (CC BY-NC 4.0) for the dataset, with accompanying code released under the Apache License 2.0. All datasets are used in compliance with their respective licenses and terms of use.

% Fourth, the theoretical analysis in this paper centers on formalizing the achievable goal distribution within the multi-objective GCSL framework, which constitutes a novel conceptual and methodological advancement for MOGCSL as well as the general GCSL problem. Nonetheless, additional formal analysis, such as theoretical comparisons with other multi-objective learning methods, could further clarify the benefits and limitations of the proposed framework.

%% file: sections/checklist.tex
\section*{NeurIPS Paper Checklist}

%%% END INSTRUCTIONS %%%

\begin{enumerate}

\item {\bf Claims}
    \item[] Question: Do the main claims made in the abstract and introduction accurately reflect the paper's contributions and scope?
    \item[] Answer: \answerYes{} % Replace by \answerYes{}, \answerNo{}, or \answerNA{}.
    \item[] Justification: We proposed a novel framework for multi-objective recommendation within goal-conditional supervised learning framework.
    \item[] Guidelines:
    \begin{itemize}
        \item The answer \answerNA{} means that the abstract and introduction do not include the claims made in the paper.
        \item The abstract and/or introduction should clearly state the claims made, including the contributions made in the paper and important assumptions and limitations. A \answerNo{} or \answerNA{} answer to this question will not be perceived well by the reviewers. 
        \item The claims made should match theoretical and experimental results, and reflect how much the results can be expected to generalize to other settings. 
        \item It is fine to include aspirational goals as motivation as long as it is clear that these goals are not attained by the paper. 
    \end{itemize}

\item {\bf Limitations}
    \item[] Question: Does the paper discuss the limitations of the work performed by the authors?
    \item[] Answer: \answerYes{} % Replace by \answerYes{}, \answerNo{}, or \answerNA{}.
    \item[] Justification: We discussed the limitations and future works to address them in the Appendix.
    \item[] Guidelines:
    \begin{itemize}
        \item The answer \answerNA{} means that the paper has no limitation while the answer \answerNo{} means that the paper has limitations, but those are not discussed in the paper. 
        \item The authors are encouraged to create a separate ``Limitations'' section in their paper.
        \item The paper should point out any strong assumptions and how robust the results are to violations of these assumptions (e.g., independence assumptions, noiseless settings, model well-specification, asymptotic approximations only holding locally). The authors should reflect on how these assumptions might be violated in practice and what the implications would be.
        \item The authors should reflect on the scope of the claims made, e.g., if the approach was only tested on a few datasets or with a few runs. In general, empirical results often depend on implicit assumptions, which should be articulated.
        \item The authors should reflect on the factors that influence the performance of the approach. For example, a facial recognition algorithm may perform poorly when image resolution is low or images are taken in low lighting. Or a speech-to-text system might not be used reliably to provide closed captions for online lectures because it fails to handle technical jargon.
        \item The authors should discuss the computational efficiency of the proposed algorithms and how they scale with dataset size.
        \item If applicable, the authors should discuss possible limitations of their approach to address problems of privacy and fairness.
        \item While the authors might fear that complete honesty about limitations might be used by reviewers as grounds for rejection, a worse outcome might be that reviewers discover limitations that aren't acknowledged in the paper. The authors should use their best judgment and recognize that individual actions in favor of transparency play an important role in developing norms that preserve the integrity of the community. Reviewers will be specifically instructed to not penalize honesty concerning limitations.
    \end{itemize}

\item {\bf Theory assumptions and proofs}
    \item[] Question: For each theoretical result, does the paper provide the full set of assumptions and a complete (and correct) proof?
    \item[] Answer: \answerYes{} % Replace by \answerYes{}, \answerNo{}, or \answerNA{}.
    \item[] Justification: The paper includes a theoretical analysis to show the property of the distribution of the achievable multi-dimensional goals of GCSL.
    \item[] Guidelines:
    \begin{itemize}
        \item The answer \answerNA{} means that the paper does not include theoretical results. 
        \item All the theorems, formulas, and proofs in the paper should be numbered and cross-referenced.
        \item All assumptions should be clearly stated or referenced in the statement of any theorems.
        \item The proofs can either appear in the main paper or the supplemental material, but if they appear in the supplemental material, the authors are encouraged to provide a short proof sketch to provide intuition. 
        \item Inversely, any informal proof provided in the core of the paper should be complemented by formal proofs provided in appendix or supplemental material.
        \item Theorems and Lemmas that the proof relies upon should be properly referenced. 
    \end{itemize}

    \item {\bf Experimental result reproducibility}
    \item[] Question: Does the paper fully disclose all the information needed to reproduce the main experimental results of the paper to the extent that it affects the main claims and/or conclusions of the paper (regardless of whether the code and data are provided or not)?
    \item[] Answer: \answerYes{} % Replace by \answerYes{}, \answerNo{}, or \answerNA{}.
    \item[] Justification: We elaborate on the experiment details in the Experiment section, including datasets, baselines, hyper-parameters, etc. We also release the code and data for reproducibility.
    \item[] Guidelines:
    \begin{itemize}
        \item The answer \answerNA{} means that the paper does not include experiments.
        \item If the paper includes experiments, a \answerNo{} answer to this question will not be perceived well by the reviewers: Making the paper reproducible is important, regardless of whether the code and data are provided or not.
        \item If the contribution is a dataset and\slash or model, the authors should describe the steps taken to make their results reproducible or verifiable. 
        \item Depending on the contribution, reproducibility can be accomplished in various ways. For example, if the contribution is a novel architecture, describing the architecture fully might suffice, or if the contribution is a specific model and empirical evaluation, it may be necessary to either make it possible for others to replicate the model with the same dataset, or provide access to the model. In general. releasing code and data is often one good way to accomplish this, but reproducibility can also be provided via detailed instructions for how to replicate the results, access to a hosted model (e.g., in the case of a large language model), releasing of a model checkpoint, or other means that are appropriate to the research performed.
        \item While NeurIPS does not require releasing code, the conference does require all submissions to provide some reasonable avenue for reproducibility, which may depend on the nature of the contribution. For example
        \begin{enumerate}
            \item If the contribution is primarily a new algorithm, the paper should make it clear how to reproduce that algorithm.
            \item If the contribution is primarily a new model architecture, the paper should describe the architecture clearly and fully.
            \item If the contribution is a new model (e.g., a large language model), then there should either be a way to access this model for reproducing the results or a way to reproduce the model (e.g., with an open-source dataset or instructions for how to construct the dataset).
            \item We recognize that reproducibility may be tricky in some cases, in which case authors are welcome to describe the particular way they provide for reproducibility. In the case of closed-source models, it may be that access to the model is limited in some way (e.g., to registered users), but it should be possible for other researchers to have some path to reproducing or verifying the results.
        \end{enumerate}
    \end{itemize}

\item {\bf Open access to data and code}
    \item[] Question: Does the paper provide open access to the data and code, with sufficient instructions to faithfully reproduce the main experimental results, as described in supplemental material?
    \item[] Answer: \answerYes{}{} % Replace by \answerYes{}, \answerNo{}, or \answerNA{}.
    \item[] Justification: Our code and data are accessible at \url{https://anonymous.4open.science/r/MOGCSL-D7A2}.
    \item[] Guidelines:
    \begin{itemize}
        \item The answer \answerNA{} means that paper does not include experiments requiring code.
        \item Please see the NeurIPS code and data submission guidelines (\url{https://neurips.cc/public/guides/CodeSubmissionPolicy}) for more details.
        \item While we encourage the release of code and data, we understand that this might not be possible, so \answerNo{} is an acceptable answer. Papers cannot be rejected simply for not including code, unless this is central to the contribution (e.g., for a new open-source benchmark).
        \item The instructions should contain the exact command and environment needed to run to reproduce the results. See the NeurIPS code and data submission guidelines (\url{https://neurips.cc/public/guides/CodeSubmissionPolicy}) for more details.
        \item The authors should provide instructions on data access and preparation, including how to access the raw data, preprocessed data, intermediate data, and generated data, etc.
        \item The authors should provide scripts to reproduce all experimental results for the new proposed method and baselines. If only a subset of experiments are reproducible, they should state which ones are omitted from the script and why.
        \item At submission time, to preserve anonymity, the authors should release anonymized versions (if applicable).
        \item Providing as much information as possible in supplemental material (appended to the paper) is recommended, but including URLs to data and code is permitted.
    \end{itemize}

\item {\bf Experimental setting/details}
    \item[] Question: Does the paper specify all the training and test details (e.g., data splits, hyperparameters, how they were chosen, type of optimizer) necessary to understand the results?
    \item[] Answer: \answerYes{} % Replace by \answerYes{}, \answerNo{}, or \answerNA{}.
    \item[] Justification: We include dataset information and training details in the Experiment part and Appendix. We also clearly describe the knowledge base and formally introduce the task settings.
    \item[] Guidelines:
    \begin{itemize}
        \item The answer \answerNA{} means that the paper does not include experiments.
        \item The experimental setting should be presented in the core of the paper to a level of detail that is necessary to appreciate the results and make sense of them.
        \item The full details can be provided either with the code, in appendix, or as supplemental material.
    \end{itemize}

\item {\bf Experiment statistical significance}
    \item[] Question: Does the paper report error bars suitably and correctly defined or other appropriate information about the statistical significance of the experiments?
    \item[] Answer: \answerYes{} % Replace by \answerYes{}, \answerNo{}, or \answerNA{}.
    \item[] Justification: For each experiment, the mean and standard deviation over 5 seeds are reported.
    \item[] Guidelines:
    \begin{itemize}
        \item The answer \answerNA{} means that the paper does not include experiments.
        \item The authors should answer \answerYes{} if the results are accompanied by error bars, confidence intervals, or statistical significance tests, at least for the experiments that support the main claims of the paper.
        \item The factors of variability that the error bars are capturing should be clearly stated (for example, train/test split, initialization, random drawing of some parameter, or overall run with given experimental conditions).
        \item The method for calculating the error bars should be explained (closed form formula, call to a library function, bootstrap, etc.)
        \item The assumptions made should be given (e.g., Normally distributed errors).
        \item It should be clear whether the error bar is the standard deviation or the standard error of the mean.
        \item It is OK to report 1-sigma error bars, but one should state it. The authors should preferably report a 2-sigma error bar than state that they have a 96\% CI, if the hypothesis of Normality of errors is not verified.
        \item For asymmetric distributions, the authors should be careful not to show in tables or figures symmetric error bars that would yield results that are out of range (e.g., negative error rates).
        \item If error bars are reported in tables or plots, the authors should explain in the text how they were calculated and reference the corresponding figures or tables in the text.
    \end{itemize}

\item {\bf Experiments compute resources}
    \item[] Question: For each experiment, does the paper provide sufficient information on the computer resources (type of compute workers, memory, time of execution) needed to reproduce the experiments?
    \item[] Answer: \answerYes{} % Replace by \answerYes{}, \answerNo{}, or \answerNA{}.
    \item[] Justification: {We run our experiments on a single NVIDIA RTX 3090 and AMD 3960X.}
    \item[] Guidelines:
    \begin{itemize}
        \item The answer \answerNA{} means that the paper does not include experiments.
        \item The paper should indicate the type of compute workers CPU or GPU, internal cluster, or cloud provider, including relevant memory and storage.
        \item The paper should provide the amount of compute required for each of the individual experimental runs as well as estimate the total compute. 
        \item The paper should disclose whether the full research project required more compute than the experiments reported in the paper (e.g., preliminary or failed experiments that didn't make it into the paper). 
    \end{itemize}
    
\item {\bf Code of ethics}
    \item[] Question: Does the research conducted in the paper conform, in every respect, with the NeurIPS Code of Ethics \url{https://neurips.cc/public/EthicsGuidelines}?
    \item[] Answer: \answerYes{} % Replace by \answerYes{}, \answerNo{}, or \answerNA{}.
    \item[] Justification: We do not induce any potential research harm mentioned in NeurIPS Code of Ethics in our paper. 
    \item[] Guidelines:
    \begin{itemize}
        \item The answer \answerNA{} means that the authors have not reviewed the NeurIPS Code of Ethics.
        \item If the authors answer \answerNo, they should explain the special circumstances that require a deviation from the Code of Ethics.
        \item The authors should make sure to preserve anonymity (e.g., if there is a special consideration due to laws or regulations in their jurisdiction).
    \end{itemize}

\item {\bf Broader impacts}
    \item[] Question: Does the paper discuss both potential positive societal impacts and negative societal impacts of the work performed?
    \item[] Answer: \answerYes{} % Replace by \answerYes{}, \answerNo{}, or \answerNA{}.
    \item[] Justification: We discussed the impact in the Appendix.
    \item[] Guidelines:
    \begin{itemize}
        \item The answer \answerNA{} means that there is no societal impact of the work performed.
        \item If the authors answer \answerNA{} or \answerNo, they should explain why their work has no societal impact or why the paper does not address societal impact.
        \item Examples of negative societal impacts include potential malicious or unintended uses (e.g., disinformation, generating fake profiles, surveillance), fairness considerations (e.g., deployment of technologies that could make decisions that unfairly impact specific groups), privacy considerations, and security considerations.
        \item The conference expects that many papers will be foundational research and not tied to particular applications, let alone deployments. However, if there is a direct path to any negative applications, the authors should point it out. For example, it is legitimate to point out that an improvement in the quality of generative models could be used to generate Deepfakes for disinformation. On the other hand, it is not needed to point out that a generic algorithm for optimizing neural networks could enable people to train models that generate Deepfakes faster.
        \item The authors should consider possible harms that could arise when the technology is being used as intended and functioning correctly, harms that could arise when the technology is being used as intended but gives incorrect results, and harms following from (intentional or unintentional) misuse of the technology.
        \item If there are negative societal impacts, the authors could also discuss possible mitigation strategies (e.g., gated release of models, providing defenses in addition to attacks, mechanisms for monitoring misuse, mechanisms to monitor how a system learns from feedback over time, improving the efficiency and accessibility of ML).
    \end{itemize}
    
\item {\bf Safeguards}
    \item[] Question: Does the paper describe safeguards that have been put in place for responsible release of data or models that have a high risk for misuse (e.g., pre-trained language models, image generators, or scraped datasets)?
    \item[] Answer: \answerNA{} % Replace by \answerYes{}, \answerNo{}, or \answerNA{}.
    \item[] Justification: Our method does not involve such safe issues. 
    \item[] Guidelines:
    \begin{itemize}
        \item The answer \answerNA{} means that the paper poses no such risks.
        \item Released models that have a high risk for misuse or dual-use should be released with necessary safeguards to allow for controlled use of the model, for example by requiring that users adhere to usage guidelines or restrictions to access the model or implementing safety filters. 
        \item Datasets that have been scraped from the Internet could pose safety risks. The authors should describe how they avoided releasing unsafe images.
        \item We recognize that providing effective safeguards is challenging, and many papers do not require this, but we encourage authors to take this into account and make a best faith effort.
    \end{itemize}

\item {\bf Licenses for existing assets}
    \item[] Question: Are the creators or original owners of assets (e.g., code, data, models), used in the paper, properly credited and are the license and terms of use explicitly mentioned and properly respected?
    \item[] Answer: \answerYes{} % Replace by \answerYes{}, \answerNo{}, or \answerNA{}.
    \item[] Justification: Creators or original owners of assets mentioned in the paper are properly cited and the license and terms of use are respected.
    \item[] Guidelines:
    \begin{itemize}
        \item The answer \answerNA{} means that the paper does not use existing assets.
        \item The authors should cite the original paper that produced the code package or dataset.
        \item The authors should state which version of the asset is used and, if possible, include a URL.
        \item The name of the license (e.g., CC-BY 4.0) should be included for each asset.
        \item For scraped data from a particular source (e.g., website), the copyright and terms of service of that source should be provided.
        \item If assets are released, the license, copyright information, and terms of use in the package should be provided. For popular datasets, \url{paperswithcode.com/datasets} has curated licenses for some datasets. Their licensing guide can help determine the license of a dataset.
        \item For existing datasets that are re-packaged, both the original license and the license of the derived asset (if it has changed) should be provided.
        \item If this information is not available online, the authors are encouraged to reach out to the asset's creators.
    \end{itemize}

\item {\bf New assets}
    \item[] Question: Are new assets introduced in the paper well documented and is the documentation provided alongside the assets?
    \item[] Answer: \answerNA{} % Replace by \answerYes{}, \answerNo{}, or \answerNA{}.
    \item[] Justification: The paper does not release new assets.
    \item[] Guidelines:
    \begin{itemize}
        \item The answer \answerNA{} means that the paper does not release new assets.
        \item Researchers should communicate the details of the dataset\slash code\slash model as part of their submissions via structured templates. This includes details about training, license, limitations, etc. 
        \item The paper should discuss whether and how consent was obtained from people whose asset is used.
        \item At submission time, remember to anonymize your assets (if applicable). You can either create an anonymized URL or include an anonymized zip file.
    \end{itemize}

\item {\bf Crowdsourcing and research with human subjects}
    \item[] Question: For crowdsourcing experiments and research with human subjects, does the paper include the full text of instructions given to participants and screenshots, if applicable, as well as details about compensation (if any)? 
    \item[] Answer: \answerNA{} % Replace by \answerYes{}, \answerNo{}, or \answerNA{}.
    \item[] Justification: The paper does not involve crowdsourcing nor research with human subjects.
    \item[] Guidelines:
    \begin{itemize}
        \item The answer \answerNA{} means that the paper does not involve crowdsourcing nor research with human subjects.
        \item Including this information in the supplemental material is fine, but if the main contribution of the paper involves human subjects, then as much detail as possible should be included in the main paper. 
        \item According to the NeurIPS Code of Ethics, workers involved in data collection, curation, or other labor should be paid at least the minimum wage in the country of the data collector. 
    \end{itemize}

\item {\bf Institutional review board (IRB) approvals or equivalent for research with human subjects}
    \item[] Question: Does the paper describe potential risks incurred by study participants, whether such risks were disclosed to the subjects, and whether Institutional Review Board (IRB) approvals (or an equivalent approval/review based on the requirements of your country or institution) were obtained?
    \item[] Answer: \answerNA{} % Replace by \answerYes{}, \answerNo{}, or \answerNA{}.
    \item[] Justification: The paper does not involve crowdsourcing nor research with human subjects.
    \item[] Guidelines:
    \begin{itemize}
        \item The answer \answerNA{} means that the paper does not involve crowdsourcing nor research with human subjects.
        \item Depending on the country in which research is conducted, IRB approval (or equivalent) may be required for any human subjects research. If you obtained IRB approval, you should clearly state this in the paper. 
        \item We recognize that the procedures for this may vary significantly between institutions and locations, and we expect authors to adhere to the NeurIPS Code of Ethics and the guidelines for their institution. 
        \item For initial submissions, do not include any information that would break anonymity (if applicable), such as the institution conducting the review.
    \end{itemize}

\item {\bf Declaration of LLM usage}
    \item[] Question: Does the paper describe the usage of LLMs if it is an important, original, or non-standard component of the core methods in this research? Note that if the LLM is used only for writing, editing, or formatting purposes and does \emph{not} impact the core methodology, scientific rigor, or originality of the research, declaration is not required.
    %this research? 
    \item[] Answer: \answerNA{} % Replace by \answerYes{}, \answerNo{}, or \answerNA{}.
    \item[] Justification: This paper does not involve usage of LLMs as component of the core methods.
    \item[] Guidelines:
    \begin{itemize}
        \item The answer \answerNA{} means that the core method development in this research does not involve LLMs as any important, original, or non-standard components.
        \item Please refer to our LLM policy in the NeurIPS handbook for what should or should not be described.
    \end{itemize}

\end{enumerate}